\RequirePackage{fix-cm}
\documentclass[smallextended]{svjour3}       
\smartqed
\usepackage{geometry}                
\usepackage{amsmath,fourier}

\usepackage[parfill]{parskip}    
\usepackage{graphicx}
\usepackage{amssymb}
\usepackage[]{algorithm2e}
\usepackage[]{algorithmic}
\usepackage{epigraph}
\usepackage{epstopdf}
\usepackage{subfig}
\usepackage{url}
\usepackage{array}
\usepackage{color}
\usepackage{tikz}
\usepackage{float}
\usepackage{multirow}

\newcommand{\tr}{\mathbf{trace}}
\newcommand{\E}{\mathbf{E}}
\newcommand{\Bcal}{\mathcal{B}}

\newcommand{\conv}{\mathbf{conv}}

\newcommand{\Scal}{\mathcal{S}}
\newcommand{\Ncal}{\mathcal{N}}
\newcommand{\Mcal}{\mathcal{M}}

\newcommand{\BQP}{\mathbf{BQP}}

\newcommand{\Fcal}{\mathcal{F}}
\newcommand{\Qcal}{\mathcal{Q}}

\newcommand{\Rbb}{\mathbb{R}}
\newcommand{\Rcal}{\mathcal{R}}
\newcommand{\sgn}{\mathbf{sign}}

\newcommand{\diag}{\mathbf{diag}}

\newcommand{\Proj}{\mathbf{Proj}}

\newcommand{\ri}{\mathbf{ri}}

\newcommand{\exclude}[1]{}

\begin{document}
\title{Regularization vs. Relaxation: A conic optimization perspective of statistical variable selection
\thanks{The first author is supported by the Washington State University new faculty seed grant; the second author is partially supported by the Simons Foundation Award 359494; the final author  is supported in part  by the U.S. Department of Energy, Office of Science, Office of Advanced Scientific Computing Research, Applied Mathematics program under contract number DE-AC02-06CH11357.}
}
\author{Hongbo Dong \and Kun Chen \and Jeff Linderoth}

\institute{Hongbo Dong \at
Department of Mathematics, Washington State University, Pullman, WA 99163 \\
\email{hongbo.dong@wsu.edu} 
\and
Kun Chen \at
Department of Statistics, University of Connecticut, Storrs, CT 06269 \\
\email{kun.chen@uconn.edu} 
\and 
Jeff Linderoth \at
Department of Industrial and Systems Engineering, University of Wisconsin-Madison, Madison, WI 53706 \\
\email{linderoth@wisc.edu}
}

\date{Initial version: 05/28/2015, current version: 10/15/2015}                                           
\maketitle

\begin{abstract}
  Variable selection is a fundamental task in statistical data analysis. Sparsity-inducing regularization methods are a popular class of methods that simultaneously perform variable selection and model estimation. The central problem is a quadratic optimization problem with an $\ell_0$-norm penalty.  Exactly enforcing the $\ell_0$-norm penalty is computationally intractable for larger scale problems, so different sparsity-inducing penalty functions that approximate the $\ell_0$-norm have been introduced.  In this paper, we show that viewing the problem from a convex relaxation perspective offers new insights.  In particular, we show that a popular sparsity-inducing concave penalty function known as the Minimax Concave Penalty (MCP), {and the reverse Huber penalty derived in a recent work 
 by Pilanci, Wainwright and Ghaoui}, can both be derived as special cases of a lifted convex relaxation called the perspective relaxation.  The optimal perspective 
 relaxation is a related minimax problem that balances the overall convexity and tightness of approximation to the $\ell_0$ norm. We show it can be solved by a semidefinite relaxation.  
 Moreover, a probabilistic interpretation of the semidefinite relaxation reveals connections with the boolean quadric polytope in combinatorial optimization. Finally by reformulating the $\ell_0$-norm penalized problem as a two-level problem, with the inner
 level being a Max-Cut problem, our proposed semidefinite relaxation can be realized by replacing the inner level problem with its semidefinite relaxation studied by Goemans and Williamson. This interpretation suggests using the Goemans-Williamson
 rounding procedure to find approximate solutions to the $\ell_0$-norm penalized problem. Numerical experiments demonstrate the tightness of our proposed semidefinite relaxation, and the effectiveness of finding approximate solutions by 
  Goemans-Williamson rounding.
\keywords{Sparse linear regression \and convex relaxation  \and semidefinite programming \and minimax concave penalty}
\textbf{Mathematics Subject Classification}  90C22, 90C47, 62J07
\end{abstract}


\section{Introduction}\label{sec:intro}
In this paper we focus on the following optimization problem with a cardinality term, that is fundamental in variable selection in linear regression model, and compressed
sensing in signal processing, 
\begin{equation}\label{eq:l0}
\zeta_{L0} := \min_{\beta} \ \ \frac{1}{2}\left\|X \beta - y \right\|_2^2 + \lambda \left\|\beta\right\|_0,
\tag{$L_0$}
\end{equation}
where $\|\cdot\|_0$, usually called the $\ell_0$-norm, denotes the number of non-zero entries in the vector under
consideration.
We primarily focus on the application of variable selection, and use the notation in statistics, where $X \in \Rbb^{n\times p}$ and $y \in \Rbb^n$ are data matrices. Each row of $X$ and the corresponding entry in $y$ is a realization of predictor variables and the associated response variable. The goal is to select a set of predictor variables to construct a linear model, with balanced model complexity (number of non-zeros in $\beta$) and model goodness of fit. Here $\lambda \geq 0$ is a tuning parameter controlling the amount of penalization on the model complexity. In practice, the best choice of $\lambda$ is not known in advance and practitioners are typically interested in 
the optimal solutions $\beta^*(\lambda)$ for all $ \lambda \in [0, +\infty)$. 


In contemporary statistical research, regression models with a large number of predictors are routinely formulated. The celebrated penalized likelihood approaches, capable of simultaneous dimension reduction and model estimation, have undergone exciting developments in recent years. These approaches typically solve an approximation to (\ref{eq:l0}) of the following form:
\begin{equation}\label{eq:approx-l0}
  \min_{\beta} \ \ \frac{1}{2}\left\|X \beta - y \right\|_2^2 +  \sum_{i} \rho(\beta_i; \lambda_i),
  \tag{$L_0$-approx}
\end{equation}
where $\rho(\cdot; \cdot)$ is a penalty function designed to induce sparsity of an optimal solution $\beta^*$, and $\left\{\lambda_i\right\}$ are some other tuning parameters that control the shape of each of such penalty functions. The design of penalty functions, optimization algorithms for solving \eqref{eq:approx-l0}, and the properties of the resulting estimators have been extensively studied in the statistical literature.  Popular methods include the lasso \cite{tib1996}, the adaptive lasso \cite{zou2006,huang2008}, the group lasso \cite{yuan2006}, the elastic net \cite{zou2005,zou2009}, the smoothly clipped absolute deviation (SCAD) penalty \cite{fan2001}, the bridge regression \cite{Frank1993,huangma2008}, the minimax concave penalty (MCP) \cite{Zhang2010} and the smooth integration of counting and absolute deviation (SICA) penalty \cite{LvFan2009,FanLv2011}. Several algorithms have been developed to solve the lasso problem and its variants, e.g. the least angle regression algorithm \cite{efron2004lars} and the coordinate descent algorithm \cite{tseng,fried2007}. For optimizing a nonconvex penalized likelihood, Fan and Li proposed an iterative local quadratic approximation (LQA). Zou and Li in \cite{zou2008} developed an iterative algorithm based on local linear approximation (LLA), which was shown to be the best convex minorization-maximization (MM) algorithm \cite{lange2004}. These local approximation approaches are commonly coupled with coordinate descent to solve general penalized likelihood problems \cite{WangLeng2007,breheny2011}. For a comprehensive account of these approaches from a statistical perspective, see \cite{buhlmanngeer2009}, \cite{Fanlv2010} and \cite{huang2012}. 

Directly solving the nonconvex problem (\ref{eq:l0}) has also received attention from the optimization community \cite{Bienstock96,Bertsimas_Shioda_2007,BertsimasKingMazumder2014}. 
The authors in \cite{FengMitchellPang2015} show promising computational results by formulating (\ref{eq:l0}) as a nonlinear program
with complementarity conditions, using nonlinear optimization algorithms to find good feasible solutions.  Recently, Bertsimas, King, and Mazumder \cite{BertsimasKingMazumder2014} demonstrate significant computational gains by exploiting
modern optimization techniques to solve various statistical problems including (\ref{eq:l0}). 
Specifically, they show that with properly-engineered techniques from mixed-integer quadratic programming, (\ref{eq:l0}) can be solved exactly for some instances of practical size. Very recently, in a more general framework, Pilanci, Wainwright and Ghaoui
\cite{PilanciWainwrightGhaoui2015} reformulated (\ref{eq:l0}), as well as its cardinality-constrained version, 
into a convex nonlinear optimization problem with binary variables. They developed conic relaxations and showed that these relaxations outperform the 
classical lasso in solution quality on both simulated and real data. Our work is very relevant to \cite{PilanciWainwrightGhaoui2015}. Indeed, we show at 
in section \ref{sec:pp} that the main convex relaxation considered in \cite{PilanciWainwrightGhaoui2015} can be derived directly as a special case of the \textit{perspective relaxation}. In section \ref{sec:SDP}
 we construct a convex relaxation that is no weaker than \textit{any} perspective relaxation.


Our goal in this paper is to show that by taking a mixed-integer quadratic optimization perspective of (\ref{eq:l0}), modern convex relaxation techniques, especially those based on conic optimization (see, e.g., papers in \cite{AnjosLasserre2012}), can 
bring new insights to develop polynomial-time variable selection methods. In section \ref{sec:convrelax} we develop the main construction 
of two convex relaxations. Section \ref{sec:pp} studies the perspective relaxation\cite{frangioni.gentile:06,Gunluk_Linderoth_2010,gunluk.linderoth:12}. 
We show that two penalty functions, the minimax concave penalty (MCP) proposed in \cite{Zhang2010} and reverse Huber penalty derived in 
\cite{PilanciWainwrightGhaoui2015}, can both be seen as special cases of perspective relaxation.
A probabilistic interpretation of the semidefinite relaxation is given in section \ref{sec:prob}, which leads to an interpretation of the matrix variable in our proposed semidefinite relaxation as the second moment of a random vector.
In section \ref{sec:GWrounding}, we show (\ref{eq:l0}) versus our proposed semidefinite relaxation is analogous to the Max-Cut problem versus its semidefinite relaxation studied by Goemans and Williamson \cite{GoWi94}.
This interpretation suggests the usage of Goemans-Williamson rounding procedure to find approximate solutions to (\ref{eq:l0}). Finally, preliminary computational experiments demonstrate the tightness of our proposed semidefinite relaxation,
and the effectiveness of finding approximate solutions (\ref{eq:l0}) with Goemans-Williamson rounding.

In this paper the space of $n\times n$ real symmetric matrices is denoted by $\Scal^n$, and the space of $n\times p$ real matrices is denoted as $\Rbb^{n\times p}$. 
The inner product between two matrices $A, B\in \Rbb^{n\times p}$ is defined as $\langle A,B \rangle = \tr(AB^T)$.  
Given a matrix $X \in \Scal^p$, we say $X \succ (\succeq) 0$ if it is positive (semi)definite. The cones of positive semidefinite matrices and positive definite matrices are denoted as $\Scal^{p}_+
:= \{X \in \Scal^p| X\succeq 0\}$ and $\Scal^{p}_{++}
:= \{X \in \Scal^p| X\succ 0\}$, respectively. The matrix $I$ is the identity matrix, and $e$ are used to denote vectors with all entries equal 1, of a conformal dimension. For a vector $\delta \in \Rbb^p$, $\diag(\delta)$
is a $p\times p$ diagonal matrix whose diagonal entries are entries of $\delta$.


\section{Convex Relaxations using Conic Optimization}\label{sec:convrelax}
The Big-M method is often used to reformulate (\ref{eq:l0}) into a (convex) mixed-integer quadratic programming problem that can be solved to optimality using branch-and-bound algorithms. As one 
motivation for our later construction, we illustrate that the classical $\ell_1$ approximation, or the lasso, is equivalent to a continuous relaxation of the big-M reformulation. In the rest of our paper 
we focus on the case that $\lambda$ is strictly positive, as the other case $\lambda = 0$ is well-understood.

Note that for any fixed $\lambda > 0$ and $M>0$ sufficiently large, (\ref{eq:l0}) is equivalent to
\begin{equation}\label{eq:MIQP}
\min_{\beta, z} \ \ \frac{1}{2}\left\|X \beta - y \right\|_2^2 + \lambda \sum_{i} z_i, \ \ \mbox{s.t.} \ \ |\beta_i| \leq M z_i, \ \ z_i \in \{0,1\}, \forall 1\leq i\leq p. 
\tag{$\mbox{MIQP}_{\lambda,M}$}
\end{equation}
Because $z$ can take only finitely-many ($2^p$) possible values, $M$ can be chosen to be independent 
of $\lambda$ (but dependent on problem data $X$ and $y$).  Specifically, let $S \subseteq \{1,...,n\}$ and  $\beta^*_S$ be an optimal solution to the linear regression in a subspace 
\[
\beta_S^* \in \arg\min_{\beta} \left\{ \ \frac{1}{2} \|X \beta - y\|_2^2 \ \middle| \ \beta_i = 0, \forall i \notin S\right\},
\]
then if we choose $M$ large enough such that
\[
M > \max_{ z \in \{0,1\}^{p} }\left\|\beta^*_{z}\right\|_\infty,
\]
an optimal solution to (\ref{eq:MIQP}) is also optimal to (\ref{eq:l0}) --- the two problems are equivalent.

Lasso \cite{tib1996} is a convex approximation to (\ref{eq:l0}) in the following form
\begin{equation}\label{eq:lasso}
\min_{\beta} \ \ \frac{1}{2}\|X\beta - y\|^2_2 + \lambda \sum_{i} |\beta_i|.  \tag{lasso} 
\end{equation}

Our first observation is that lasso can be interpreted as a special continuous relaxation of (\ref{eq:MIQP}). 
\begin{proposition}A continuous relaxation of (\ref{eq:MIQP}), where the binary conditions $z_i \in \{0,1\}$ are relaxed to $z_i \in [0,+\infty), \forall i$, is equivalent to (\ref{eq:lasso}) with penalty parameter $\bar{\lambda}$, where $\bar{\lambda} = \frac{\lambda}{M}$.
\end{proposition}
\begin{proof}
When the binary conditions $z_i \in \{0,1\}, \forall i$ are relaxed to $z_i \in [0,+\infty), \forall i$, as $\lambda > 0$, $z_i$ must take the value $\frac{|\beta_i|}{M}$ in an optimal solution to \eqref{eq:MIQP}. Therefore this continuous relaxation is equivalent to
\[
\min_{\beta} \ \ \frac{1}{2}\|X\beta - y\|^2_2 + \frac{\lambda}{M} \sum_{i} |\beta_i|. 
\]\qed
\end{proof}

This interpretation of lasso motivates us to explore the following two questions in this paper:
\begin{enumerate}
\item Do more sophisticated convex relaxation techniques for (\ref{eq:MIQP}), specifically those based on lifting and conic optimization, have connections with regularization methods proposed in statistical and machine learning community?
\item Can convex relaxations based on conic optimization bring new insights for developing methods for variable selection?
\end{enumerate}

This paper answers both questions in the affirmative.  In the remaining part of this section, we discuss two 
convex relaxations of (\ref{eq:MIQP}). The first is called the \textit{perspective relaxation}
(see, e.g., \cite{Frangioni_Gentile_2007,Gunluk_Linderoth_2010,gunluk.linderoth:12}), which is 
a second-order-cone programming (SOCP) problem.  We show in section \ref{sec:pp} that the penalty form of perspective relaxation
generalizes two penalty functions in the literature, the \textit{minimax concave penalty} (MCP) \cite{Zhang2010} and the \textit{reverse Huber penalty} \cite{PilanciWainwrightGhaoui2015}.
The second convex relaxation we introduce in section \ref{sec:SDP} is based on semidefinite programming (SDP). We 
show that this convex relaxation is equivalent to minimax formulations corresponding to the \textit{optimal} perspective relaxation.


\subsection{Perspective Relaxation, Minimax Concave Penalty, and Reverse Huber Penalty}\label{sec:pp}
To start, we present a derivation of the perspective relaxation. Let $\delta \in \Rbb_+^{p}$ be a vector such that
$X^T X - \diag(\delta) \succeq 0$.  By splitting the quadratic form $X^T X = \left(X^T X - \diag(\delta)\right) + \diag(\delta)$, (\ref{eq:MIQP}) can be written as
\begin{align}
\min_{b,z} \ \ & \frac{1}{2} b^T (X^T X - \diag(\delta) ) b - (X^T y)^T b + \frac{1}{2}\sum_{i}\delta_i b_i^2 + \lambda \sum_{i} z_i + \frac{1}{2}y^T y\label{eq:extract-diag} \\
\mbox{s.t.} \ \ &  -M z_i \leq b_i \leq M z_i, \label{eq:vub} \\
& z_i \in \{0,1\}, \ \forall i . 
\end{align}
Two additional remarks are in order. First, if $X^T X$ is positive definite, then a (non-trivial) $\delta \neq 0$ can be found. Otherwise if $X^T X$ has some zero eigenvalue,  and the corresponding eigenspace contains some dense vector, then the only $\delta$  that satisfies $X^T X - \diag(\delta) \succeq 0$ is the zero vector. In this case, a meaningful perspective relaxation cannot be formulated.  In the rest of the paper, we will assume that $X^T X$ is positive definite. 
This assumption introduces some loss of generality.  From a statistical point of view, when $n\geq p$ and each row of $X$ is generated independently from a full-dimensional continuous distribution, 
$X^T X$ is guaranteed to be positive definite. However when $n < p$, i.e., there are fewer data points than the number of predictors, $X^T X$ is not full rank.
In this scenario, some modification of problem (\ref{eq:MIQP}) is necessary for our construction to be valid. A popular idea in statistics, called stabilization, is to add an additional regularization term, 
$0.5\mu\|\beta\|_2^2$ where $\mu>0$, into the objective function of (\ref{eq:MIQP}). Then the objective function becomes strictly convex, and the quadratic form becomes $X^T X + \mu I$, where $I$ is the $p\times p$ identity matrix.  This regularization term 
is also used in \cite{PilanciWainwrightGhaoui2015}.

Second, our change of notation, from $\beta$ to $b$, is intentional, in order to be consistent with the semidefinite relaxation that we 
discuss later.  In Section~\ref{sec:prob}, the variable $b$ used in our relaxations has an interpretation of the \textit{expected value} of $\beta$, which is then considered as a random vector.

By introducing additional variables $s_i$ to represent $b_i^2$, the valid perspective constraints $s_i z_i \geq b_i^2$ and letting $M\mapsto +\infty$, we obtain the perspective relaxation
\[
\begin{aligned}\label{eq:PR}
\zeta_{PR(\delta)} := \min_{b, z, s} \ \ & \frac{1}{2} b^T (X^T X - \diag(\delta) ) b - (X^T y)^T b + \frac{1}{2}\sum_{i}\delta_i s_i + 
\lambda \sum_{i} z_i + \frac{1}{2}y^T y, \\
 \mbox{s.t.} \ \ &  s_i z_i \geq b_i^2, \ s_i \geq 0, \ 0\leq  z_i \leq 1 \ \ \forall i. 
\end{aligned} \tag{$PR_{\delta}$}
\]
Note that if for all $i$, $z_i \in \{0,1\}$, then the perspective constraints $s_i z_i \geq b_i^2$ imply that $b_i \neq 0$ only when $z_i=1$.  Proposition~\ref{prop:PR_attain} shows that the minimum of (\ref{eq:PR}) is always attained, justifying the usage $``\min"$ instead of  $``\inf"$ in (\ref{eq:PR}). 

\begin{proposition}\label{prop:PR_attain}
Assume $X^T X \succ 0$, $\lambda \geq 0$, and let $\delta \in \Rbb_+^{p}$ and
$X^T X - \diag(\delta) \succeq 0$.  The optimal value of (\ref{eq:PR}) is attained at some finite point.
\end{proposition}
\begin{proof}
First observe that the objective function in (\ref{eq:PR}) can be rewritten as
\[
\frac{1}{2} \|X b - y \|_2^2 + \sum_{i} \left( \frac{1}{2} \delta_i \left(s_i-b_i^2\right) + \lambda z_i \right) \geq \frac{1}{2} \|Xb - y \|_2^2 ,
\]
where the inequality holds for any feasible $(b, z, s)$.
Now let $\hat{b} := \arg\min_{b}\|Xb - y\|_2^2$ (which is unique by the assumption $X^T X \succ 0$), and $\hat{s}_i := \hat{b}_i^2, \hat{z}_i := 1$ for all $i$, then $(\hat{b},
\hat{s}, \hat{z})$ is a feasible solution to (\ref{eq:PR}). By the strict convexity of $\|X b - y\|_2^2$, 
there exists $R > 0$ such that 
$\forall b, \|b\|_2\geq R$,
\[
\frac{1}{2} \|Xb - y \|_2^2 \geq \frac{1}{2} \|X\hat{b} - y \|_2^2 + \lambda p = 
\frac{1}{2} \|X\hat{b} - y \|_2^2 + \sum_{i} \frac{1}{2} \delta_i \left(\hat{s}_i-\hat{b}_i^2\right) + \lambda \hat{z}_i .
\]
Therefore the optimal value of (\ref{eq:PR}) must be attained at some point in $\left\{b \mid \|b\|_2 \leq R\right\}$. \qed
\end{proof}


Next we derive the \textit{penalization} form of (\ref{eq:PR}). 
\begin{theorem}\label{thm:pp}Assume $X^T X \succ 0$, $\lambda > 0$, and let $\delta \in \Rbb_+^{p}$ and
$X^T X - \diag(\delta) \succeq 0$. (\ref{eq:PR}) is equivalent to the following regularized regression problem 
\begin{equation}\label{eq:PRreg}
\min_{b} \ \frac{1}{2} \|Xb - y \|_2^2 + \sum_{i} \rho_{\delta_i} (b_i; \lambda), \tag{$PR_\delta:reg$}
\end{equation}
where 
\begin{equation}\label{eq:PRpenalty}
\rho_{\delta_i}(b_i; \lambda) = 
\begin{cases}
\sqrt{2\delta_i\lambda} |b_i| - \frac{1}{2} \delta_i b_i^2, &  \ if  \ \delta_i b_i^2 \leq 2\lambda ; \\
\lambda, & \ if \ \delta_i b_i^2 > 2\lambda.
\end{cases}
\end{equation}
\end{theorem}
\begin{proof}
Observe that 
the objective function in (\ref{eq:PR}) is 
\[
\frac{1}{2} \|Xb - y \|_2^2 + \sum_{i} \frac{1}{2} \delta_i \left(s_i-b_i^2\right) + \lambda z_i.
\]
Then (\ref{eq:PR}) can be reformulated as a regularized regression problem 
\begin{equation}\label{eq:PRreg}
\min_{b} \ \frac{1}{2} \|Xb - y \|_2^2 + \sum_{i} \rho_{\delta_i} (b_i;\lambda), \tag{$PR_\delta:reg$}
\end{equation}
where 
\begin{equation} \label{eq:rhodef}
\rho_{\delta_i} (b_i; \lambda) = \min_{s_i z_i \geq b_i^2, s_i \geq 0, z_i \in [0,1]}\frac{1}{2} \delta_i \left(s_i-b_i^2\right) + \lambda z_i.
\end{equation}
We can derive an explicit, closed form for $\rho_{\delta_i}(b_i;\lambda)$. If $\delta_i = 0$, it is easy to see that $\rho_{\delta_i} (b_i; \lambda) = 0$. We then focus on the case
$\delta_i > 0$.
When $b_i = 0$, it is again easy to see that the optimal solution to \eqref{eq:rhodef} is attained at $s_i = z_i = 0$, and
$\rho_{\delta_i} (0; \lambda) = 0$. When $b_i \neq 0$, 
by the constraint $s_i z_i \geq b_i^2$ and $z_i \in [0,1]$ we must have $s_i \geq b_i^2$, and $z_i$ must take the value $z_i = \frac{b_i^2}{s_i}$ in an optimal solution.  Therefore, the minimization problem in \eqref{eq:rhodef} becomes a one-dimensional problem
\[
\rho_{\delta_i} (b_i; \lambda) = \min_{ s_i \geq b_i^2} \frac{1}{2} \delta_i \left(s_i-b_i^2\right) + \lambda \frac{b_i^2}{s_i}.
\]
Since $\frac{1}{2} \delta_i \left(s_i-b_i^2\right) + \lambda \frac{b_i^2}{s_i}$ is a 
convex function of $s_i$ when $s_i \geq b_i^2 > 0$, its minimum is attained at $s_i = \sqrt{\frac{2\lambda b_i^2}{\delta_i}}$ when $\sqrt{\frac{2\lambda b_i^2}{\delta_i}}
\geq b_i^2$, and $s_i = b_i^2$ when $\sqrt{\frac{2\lambda b_i^2}{\delta_i}} \leq b_i^2$. Therefore
\begin{equation*}
\rho_{\delta_i}(b_i; \lambda) = 
\begin{cases}
\sqrt{2\delta_i\lambda} |b_i| - \frac{1}{2} \delta_i b_i^2, &  \ if  \ \delta_i b_i^2 \leq 2\lambda ; \\
\lambda, & \ if \ \delta_i b_i^2 > 2\lambda.
\end{cases}
\end{equation*}
Note that this formula also holds when $\delta_i = 0$ or $b_i = 0$.
\qed
\end{proof}

The penalty function (\ref{eq:PRpenalty}) is a nonconvex function of $b_i$. However, (\ref{eq:PR}), as well as (\ref{eq:PRreg}), is a convex problem as long as $X^T X - \diag(\delta)$
is positive semidefinite. Intuitively, the nonconvexity in $\rho_{\delta_i}(\cdot)$ is compensated by the (strict) convexity
of $\|Xb-y\|_2^2$.

Since (\ref{eq:PR}) is derived from a convex relaxation of a binary formulation of (\ref{eq:l0}), it 
is not a surprise that in
the equivalent penalization form, $\rho_{\delta_i} (b_i)$ is an underestimation of $\lambda \cdot \mathbb{1}_{b_i \neq 0}$, where
\[ \mathbb{1}_{t \neq 0} := \left\{ \begin{array}{cl} 1, \ \ & t \neq 0\\
  0, \ \ & \mbox{otherwise.} \end{array} \right. \]
In fact, 
it suffices to verify this for the first case in (\ref{eq:PRpenalty}). Indeed, 
\[
\rho_{\delta_i}(b_i; \lambda) = \sqrt{2\delta_i \lambda} |b_i| - \frac{1}{2}\delta_i b_i^2 = \lambda -\left(\sqrt{\frac{\delta_i}{2}}\left|b_i\right| - \sqrt{\lambda}\right)^2 \leq \lambda.
\]

\begin{remark}
In fact, the formula (\ref{eq:PRpenalty}) is a rediscovery of the Minimax Concave Penalty (MCP) proposed
by Zhang \cite{Zhang2010}.  Table \ref{tab:MCPnotation} demonstrates the translation between
our notation (parameters) and the notation used in \cite{Zhang2010} (we put a tilde over Zhang's notation to avoid confusion).
\begin{table}[htdp]
\caption{Translation of parameters between $\rho_{\delta_i} (\cdot; \cdot)$ and the MCP}
\begin{center}
\begin{tabular}{|c|c|}
\hline
Our notation & MCP \cite{Zhang2010} \\
\hline
\hline
$\delta_i$  & $1/\tilde{\gamma}$ \\
 $\lambda$ & $\frac{1}{2}\tilde{\gamma} \tilde{\lambda}^2$ \\
$\sqrt{2 \delta_i \lambda}$ & $\tilde{\lambda}$ \\
 $1/\delta_i$ & $\tilde{\gamma}$\\
 \hline
\end{tabular}
\end{center}
\label{tab:MCPnotation}
\end{table}%
Actually (\ref{eq:PRreg}) is slightly more general than MCP functions used in \cite{Zhang2010}, as Zhang uses one single parameter, $\tilde{\gamma}$, to control the concavity of every
penalty term, which corresponds to the special case of perspective relaxation where all $\delta_i$ chosen to be the same (and strictly positive).  Zhang also derived the condition for overall convexity, $X^T X - (1/\tilde{\gamma}) I \succeq 0$, which matches with our condition $X^T X - \diag(\delta) \succeq 0$. 
\end{remark}

Figure \ref{fig:MCP} illustrates the penalty function (\ref{eq:PRpenalty}) with $\lambda = 1$ and different choices of parameter $\delta_i$. With fixed $\delta_i$, this function is continuously differentiable at any 
nonzero value, and its second derivative is $-\delta_i$ when $b_i \in \left[-\sqrt{\frac{2\lambda}{\delta_i}},0\right)
\cup \left(0, \sqrt{\frac{2\lambda}{\delta_i}}\right]$. When $b_i$ is fixed, $\rho_{\delta_i}(b_i)$ is a concave function of $\delta_i$ when $\delta_i > 0$.

\begin{figure}[htdp]%
    \centering
    \subfloat[Penalty function]{{\includegraphics[width=7.1cm]{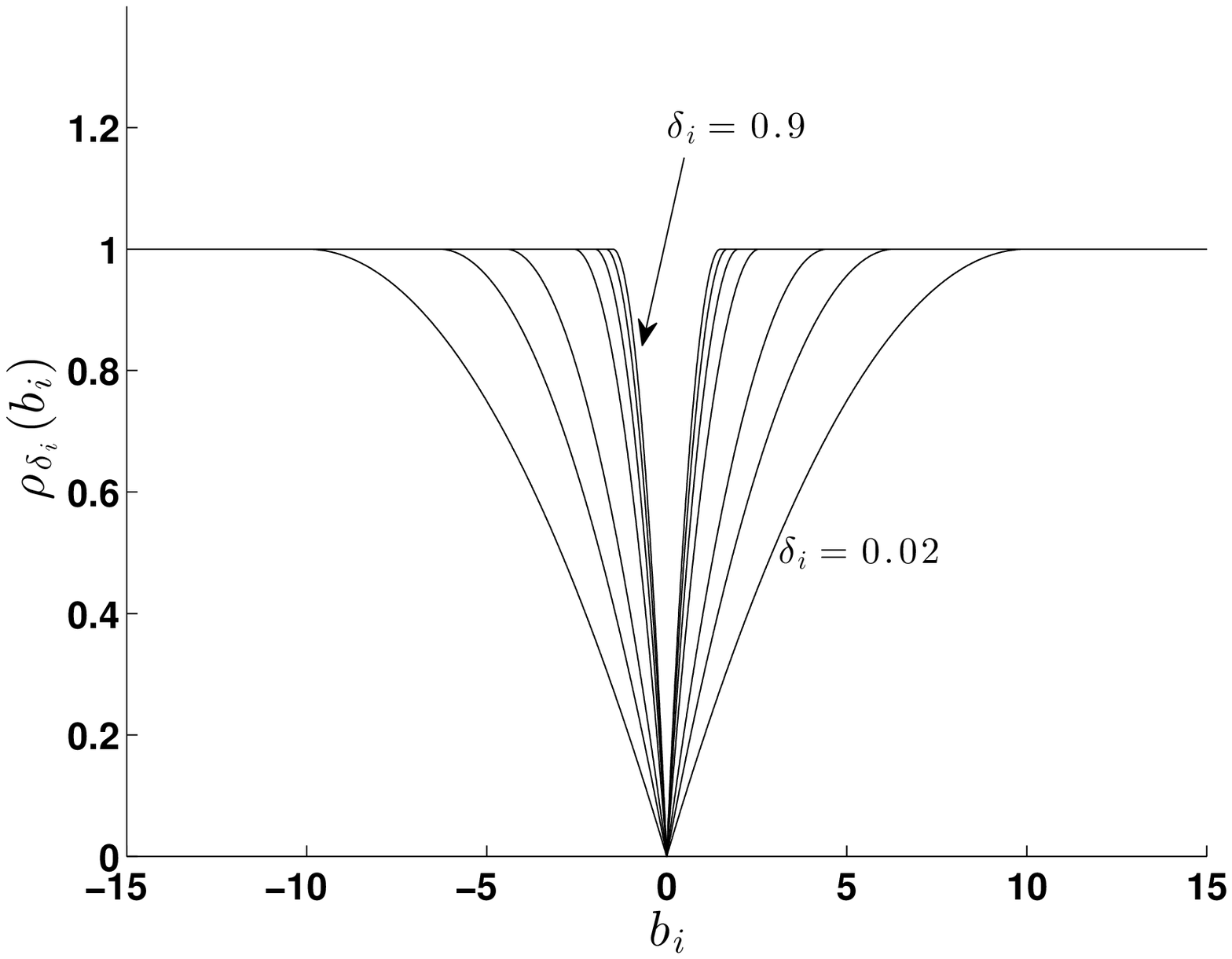} }}%
    \qquad
    \subfloat[Penalty gradient]{{\includegraphics[width=7.1cm]{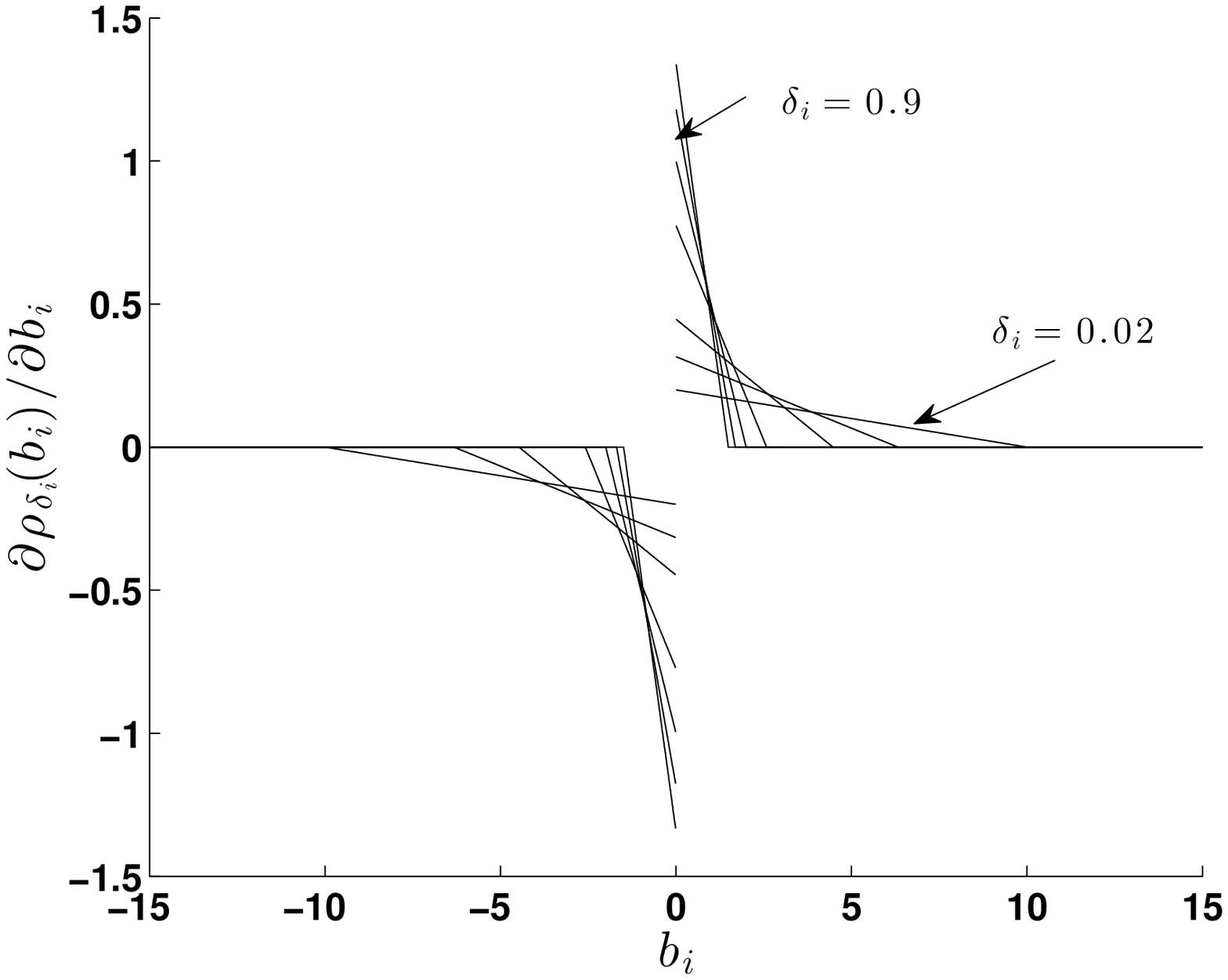} }}%
    \caption{Illustration of penalty function (\ref{eq:PRpenalty})}%
    \label{fig:MCP}%
\end{figure}

\begin{remark}\label{rmk:PWG}
We show that another convex relaxation proposed by Pilanci, Wainwright and El Ghaoui \cite{PilanciWainwrightGhaoui2015} is also a special case of perspective relaxation. They considered
the following $\ell_2$ - $\ell_0$ penalized problem with $\mu > 0$,
\begin{equation}\label{eq:l2l0}
\min_b \ \ \frac{1}{2}\|Xb - y\|_2^2 + \frac{1}{2} \mu \|b\|_2^2 + \lambda \|b\|_0, \tag{L2L0}
\end{equation}
and derived a convex relaxation\footnote{The difference of a constant factor from Corollary 3  in \cite{PilanciWainwrightGhaoui2015} is due to a typo in their derivation.},
\begin{equation}\label{eq:PWGpenalty}
\min_b \ \ \frac{1}{2} \left\|Xb - y\right\|_2^2 + 2 \lambda B\left(\sqrt{\frac{\mu}{2\lambda}} b_i\right)
\end{equation}
where $B$ denotes the reverse Huber penalty
\[
B(t)  = \begin{cases}|t| & \text{ if } |t| \leq 1, \\ \frac{t^2+1}{2} & \text{otherwise}. \end{cases}
\]

Now we derive a perspective relaxation for (\ref{eq:l2l0}) can show its equivalence to (\ref{eq:PWGpenalty}). Note that (\ref{eq:l2l0}) can be reformulated to (\ref{eq:l0}) by redefining the data matrices
\begin{equation}\label{eq:reformL2L0}
\min_b \ \ \frac{1}{2}\|\tilde{X} b - \tilde{y}\|_2^2 + \lambda \|b\|_0, \mbox{ where } \tilde{X} = \begin{bmatrix} X \\ \mu I_p\end{bmatrix}, \mbox{and} 
\tilde{y} = \begin{bmatrix} y \\ 0\end{bmatrix},
\end{equation}

Obviously we have $\tilde{X}^T \tilde{X} = X^T X + \mu I_p \succ 0$. To construct a perspective relaxation, one straightforward choice of $\delta$ is that $\delta_i = \mu, \forall i$, and by 
Theorem \ref{thm:pp} the perspective relaxation reads
\begin{align}
& \min_{b} \ \frac{1}{2} \|\tilde{X}b - \tilde{y} \|_2^2 + \sum_{i} \rho_{\mu} (b_i;\lambda) \notag \\
\Longleftrightarrow & \min_{b} \ \frac{1}{2} \|Xb - y \|_2^2 + \sum_{i} \left\{\rho_{\mu} (b_i;\lambda) + \frac{1}{2} \mu b_i^2 \right\} \label{pen:berhu}
\end{align}

We then verify that (\ref{pen:berhu}) is the same as relaxation (\ref{eq:PWGpenalty}). Indeed, 
if $\mu b_i^2 \leq 2 \lambda$, we have $ \left|\sqrt{\frac{\mu}{2\lambda}} b_i\right| \leq 1$, and 
\[
\rho_\mu(b_i; \lambda) + \frac{1}{2}\mu b_i^2 = \sqrt{2\mu\lambda} |b_i| = 2\lambda \cdot \left|\sqrt{\frac{\mu}{2\lambda}} b_i\right| = 2 \lambda B\left(\sqrt{\frac{\mu}{2\lambda}} b_i\right).
\]
Otherwise 
\[
\rho_\mu(b_i; \lambda) + \frac{1}{2}\mu b_i^2 = \lambda + \frac{1}{2}\mu b_i^2 = 2\lambda \cdot \frac{\left(\sqrt{\frac{\mu}{2\lambda}} b_i\right)^2+1}{2} 
=2 \lambda B\left(\sqrt{\frac{\mu}{2\lambda}} b_i\right).
\]

%
\end{remark}

\subsection{The Optimal Perspective Relaxation}\label{sec:SDP}
As in Theorem \ref{thm:pp}, the perspective relaxation is parameterized by a vector $\delta$, which takes value in a constrained set
$\left\{\delta \in \Rbb_+^{p}\middle| X^T X - \diag(\delta) \succeq 0\right\}$. In this section we aim to find the best parameter $\delta$ such that $\zeta_{PR(\delta)}$,
which is a lower bound of the optimal value of (\ref{eq:l0}), is as large as possible.
Intuitively, as $\delta_i \mapsto +\infty$,  $\rho_{\delta_i}(b_i;\lambda)$ converges pointwise to the indicator function $\lambda\mathbb{1}_{b_i \neq 0}$, we would like to choose the 
entries in $\delta$ large enough so that the condition $X^T X - \diag(\delta) \succeq 0$ is tight. Further  we wish to achieve the superemum of all lower bounds provided by perspective relaxations,
\begin{equation}\label{eq:maxmin}
\sup_{\delta \in \Rbb_+^{p}} \ \inf_{b} \left\{ \frac{1}{2} \|Xb - y\|_2^2 + \sum_{i} \rho_{\delta_i}(b_i;\lambda) \middle| X^T X - \diag(\delta) \succeq 0\right\}. \tag{Sup-Inf}
\end{equation} 
Alternatively one can simultaneously exploit the infinitely many penalty functions corresponding to all $\delta \in \Rbb^{p}_+$ and $X^T X - \diag(\delta) \succeq 0$, and replace the 
penalization term in (\ref{eq:PRreg}) by its \textit{pointwise superemum} over all admissible $\delta$,
\begin{equation}\label{eq:minimax}
\inf_{b} \ \sup_{\delta \in \Rbb_+^{p}}\left\{ \frac{1}{2} \|Xb - y\|_2^2 + \sum_{i} \rho_{\delta_i}(b_i;\lambda) \middle| X^T X - \diag(\delta) \succeq 0\right\}. \tag{Inf-Sup}
\end{equation}
We show that these two problems are equivalent using the minimax theory in convex analysis \cite{RockafellarCA}. Indeed one may immediately see that the optimal values of 
(\ref{eq:maxmin}) and (\ref{eq:minimax}) are the same, because of the fact that $\delta$ takes values in a compact set (Corollary 37.3.2 in \cite{RockafellarCA}). To show that a saddle point 
exists, we need the following theorem.
\begin{theorem}[Theorem 37.6 in \cite{RockafellarCA}]\label{thm:minimax}
Let $K(\delta, b)$ be a closed proper concave-convex function with effective domain $C \times D$. If both of the following conditions,
\begin{enumerate}
\item The convex functions $K(\delta, \cdot)$ for $\delta \in \ri(C)$ have no common direction of recession;
\item The convex functions $-K(\cdot, b)$ for $b \in \ri(D)$ have no common direction of recession;
\end{enumerate}
are satisfied, then $K$ has a saddle-point in $C\times D$. In other words, there exists $(\delta^*, b^*) \in C\times D$, such that 
\begin{align*}
\inf_{b \in D} \sup_{\delta \in C} K(\delta, b) &= \sup_{\delta \in C} \inf_{b \in D} K(\delta, b) = K(\delta^*, b^*).
\end{align*}
\end{theorem}
The following theorem applies Theorem \ref{thm:minimax} in our context.
\begin{theorem}
Let $\zeta_{\sup\inf}$ and $\zeta_{\inf\sup}$ be the optimal values of (\ref{eq:maxmin}) and (\ref{eq:minimax}) respectively. We have $\zeta_{\sup\inf} = \zeta_{\inf\sup}$. If 
$X^T X \succ 0$, then there exists $\delta^* \in \left\{\delta \in \Rbb^{p}_+ \middle| X^T X - \diag(\delta) \succeq 0 \right\} $ and $b^* \in \Rbb^p$
such that 
\[
\zeta_{\sup\inf} =\frac{1}{2} \|Xb^* - y\|_2^2 + \sum_{i} \rho_{\delta_i^*}(b^*_i; \lambda) = \zeta_{\inf\sup}.
\]
\end{theorem}
\begin{proof}
Define the sets $C := \left\{u\in \Rbb_+^{p} \middle| X^T X - \diag(u) \succeq 0\right\}$ and $D := \Rbb^p$, and a function on $\Rbb^{2 p}$,
\begin{equation}\label{eq:concave_convex}
K(\delta, b) := \begin{cases}
\frac{1}{2}\left\|X b - y\right\|_2^2 + \sum_{i} \rho_{\delta_i}(b_i; \lambda), &\ \ \forall  \delta \in C\\
-\infty, &\ \ \forall \delta \notin C
\end{cases}.
\end{equation}
Then $K(\delta, b)$ is concave in $\delta$ and convex in $b$, a so-called  
\textit{concave-convex function}. The \textit{effective domain} of $K(\cdot, \cdot)$ is 
\[
\left\{\delta \mid K(\delta, b) > -\infty, \forall b \right\} \times \left\{ b \mid K(\delta, b) <+\infty, \forall \delta \right\} = C \times D.
\] 
$K$ is \textit{proper}, as its effective domain is non-empty. For each fixed $b$, the function $K(\cdot, b)$ is upper-semicontinuous, and for each fixed $\delta \in C$, $K(\delta, \cdot)$ is lower-semicontinuous. Therefore $K(\delta, b)$ is both \textit{concave-closed} and \textit{convex-closed}, and is said to be \textit{closed} (e.g., section 34 \cite{RockafellarCA}).
Then $K(\cdot, \cdot)$, $C$ and $D$ satisfy assumptions in Corollary 37.3.2 in \cite{RockafellarCA}. Therefore  $\zeta_{\sup\inf} = \zeta_{\inf\sup}$. 
Further, for any $b \in \Rbb^{p}$, $-K(\cdot, b)$ has no direction of recession as 
$C$ is bounded. If $X^T X \succ 0$, for any $\delta \in \ri(C)$, the quadratic form $\|Xb-y\|_2^2$ is strictly convex, and $\rho_{\delta_i}(b_i; \lambda)$ is constant 
for $|b_i|$ sufficiently large. Therefore $K(\delta, \cdot)$ has no direction of recession, and by Theorem \ref{thm:minimax}, there exists $\delta^* \in C$ and $b^* \in D$ such that
\[
\zeta_{\sup\inf} =K(\delta^*, b^*) = \zeta_{\inf\sup},
\]
i.e., $(\delta^*,b^*)$ is a saddle point for (\ref{eq:maxmin}) and (\ref{eq:minimax}). \qed
\end{proof}

We now introduce a second convex relaxation of (\ref{eq:MIQP}), a semidefinite program (SDP), and we show that this semidefinite relaxation solves the minimax pair (\ref{eq:minimax}) and (\ref{eq:maxmin}). 

The problem (\ref{eq:MIQP}) can be equivalently formulated as the following convex problem:
\[
\min_{b, z, B} \ \ \frac{1}{2} \left\langle X^T X, B \right\rangle - y^T X b +\frac{1}{2} y^T y + \lambda \sum_{i} z_i, \ \ \mbox{s.t.}  \ \ \left(b, z, B\right) \in \conv(S_M),
\]
where set $S_M$ is defined as
\begin{equation}\label{eq:MIQP_cvx}
S_M :=  \left\{(b,z,B) \in \mathbb{R}^{2p+\frac{p(p+1)}{2}},
\begin{array}{l}
b \in [-M,M]^n, z \in \{0,1\}^{p},\\
B = b b^T, 
|b_i| \leq M z_i, \forall i\end{array}
\right\}.
\end{equation}

Convex relaxations of (\ref{eq:MIQP}) can be constructed by 
relaxing the set $\conv(S_M)$. Since $(b, z, B) \in S_M$ if and only if $(M^{-1} b, z, M^{-2} B) \in S_1$,
we may focus on convex relaxations of $S_1$. 
Furthermore, since the data matrix $X^T X$ is always 
positive semidefinite and there is no other constraint on $B$, we may replace the nonconvex condition $B = bb^T$ with the convex constraints $B \succeq bb^T$.  
Therefore, without loss of generality,
\exclude{the constraint in (\ref{eq:MIQP_cvx}) can be replaced by 
$\left(M^{-1} b, z, M^{-2} B\right) \in \conv\left(S_1^{\succeq}\right)$, where the convex hull is 
  defined as follows,}
we may seek  relaxations of the set 
\[
\conv\left(S_1^{\succeq}\right) := 
\conv \left\{(b,z,B) \in \mathbb{R}^{2p+\frac{p(p+1)}{2}},
\begin{array}{l}
b \in [-1, 1]^n, z \in \{0,1\}^{p},\\
B \succeq b b^T, 
|b_i| \leq z_i, \forall i \end{array}
\right\}.
\]
One strategy for obtaining valid inequalities for
$\conv\left(S_1^{\succeq}\right)$ is to strengthen, or lift, valid
inequalities for the set
\begin{equation}\label{eq:projConvS}
\Qcal := \left\{(b, B) \in \Rbb^{p+\frac{p(p+1)}{2}}\middle| b \in [-1,1]^n, B \succeq bb^T \right\}
 = \Proj_{(b,B)} \conv\left(S_1^{\succeq}\right).
\end{equation}
The simplest class of such lifted inequalities are probably the
perspective inequalities, $B_{ii} z_i \geq b_i^2, \forall i $, which are lifted 
from the inequalities $B_{ii} \geq b_i^2$ valid for $\Qcal$ \cite{dong.linderoth:13}. 
The following proposition shows that such lifted constraints, together with $\Qcal$, 
captures a certain class of valid linear inequalities for $\conv\left(S_1^{\succeq}\right)$.

\begin{proposition}\label{thm:sep}
Any linear inequality $\langle \Gamma, B \rangle + \alpha^T b +
\sum_{i } \gamma_i z_i \geq \tau$ valid for
$\conv\left\{S_1^{\succeq}\right\}$ that satisifies one of the following properties:
\begin{enumerate}
\item $\Gamma$ is diagonal;
\item $\gamma_i = 0$, $\forall i$;
\end{enumerate}
 is also valid for the convex set
\[  \Rcal := \left\{(b,z,B) \in [-1,1]^p \times \Rbb^{p + \frac{p(p+1)}{2}}\middle|
B \succeq b b^T, B_{ii} z_i \geq b_i^2, 0\leq z_i \leq 1, \forall i\right\}.
\]
\end{proposition}
\begin{proof}

Suppose that $\langle \Gamma, B \rangle + \alpha^T b + 
\sum_{i} \gamma_i z_i \geq \tau$ is valid for $\conv\left(S_1^{\succeq}\right)$,
and $\Gamma$ is diagonal.  We show that it is valid for $\Rcal$ with the following inequality chain:
\begin{align*}
\tau &\leq \min_{\substack{b \in [-1,1]^p, \forall i, \\ |b_i| \leq z_i \in \{0,1\}, \forall i}} \left\{\sum_{i=1}^p \Gamma_{ii} b_i^2 + \alpha_i b_i + \gamma_i z_i\right\} \\
& = \sum_{i=1}^p 
\min_{\substack{b_i \in [-1,1], \forall i, \\ |b_i| \leq z_i \in \{0,1\}, \forall i}} \left\{\Gamma_{ii} b_i^2 + \alpha_i b_i + \gamma_i z_i\right\} \\
&= \sum_{i=1}^p \min_{ \substack{b_i\in [-1,1], \forall i, \\ B_{ii} z_i \geq b_i^2, 0\leq z_i \leq 1, B_{ii} \geq 0 \forall i}} \left\{\Gamma_{ii} B_{ii} + \alpha_i b_i + \gamma_i z_i\right\}.
\end{align*}

The first equality is because of the separability of the minimization problem, while the second
equality is due to the convex hull characterization of the following set in $\Rbb^3$,
\[
\left\{\left(b_i, b_i^2, z_i\right) \middle| b_i \in [-1,1], |b_i| \leq z_i, z_i \in \{0,1\}
\right\}.
\]
See, for example, \cite{Gunluk_Linderoth_2010} for a proof.  Therefore 
$\langle \Gamma, B \rangle + \alpha^T b + \sum_{i} \gamma_i z_i \geq \delta$ must be valid for 
$\Rcal$.

On the other hand, if $\langle \Gamma, B \rangle + \alpha^T b + 
\sum_{i} \gamma_i z_i \geq \delta$ is valid for $\conv\left(S_1^{\succeq}\right)$,
and $\gamma_i = 0, \forall i$, then by equation (\ref{eq:projConvS}), the inequality is also valid
for $\Rcal$. \qed
\end{proof}

By using $\Rcal$ as a convex relaxation for $S_1^{\succeq}$, a semidefinite relaxation for 
(\ref{eq:MIQP})  is
\[
\begin{aligned}
\min_{b, z, B} \ \ & \frac{1}{2} \left\langle X^T X, B \right\rangle - y^T X b + \frac{1}{2} y^Ty+ \lambda \sum_{i} z_i, \\
\ \ \mbox{s.t.}  &   B\succeq bb^T, \begin{bmatrix} z_i & b_i \\ b_i & B_{ii}\end{bmatrix} \succeq 0,  \forall i, \\
& M^{-1} b \in [-1,1]^p,
\end{aligned}
\]
Note that $z_i \geq 0$ is implied by the 2 by 2 positive semidefinite constraints. The upper bounds $z_i \leq 1$, although not explicitly imposed, must hold in optimal solutions. This is because 
in an optimal solution, $z_i$ must take the value  0 if $B_{ii}=0$, and value $\frac{b_i^2}{B_{ii}}$ if $B_{ii}\neq 0$, while $B\succeq bb^T$ implies $B_{ii} \geq b_i^2, \forall i$.
The only constraint where $M$ does not cancel is 
$M^{-1} b \in [-1,1]^p$. We may choose to relax this constraint, or in fact we can be show that, with the assumption of 
$X^T X \succ 0$, $M$ can be chosen large enough (and independent of $\lambda$), such that this constraint is never active. To see this, let $\hat{b} \in \arg\min_b \|X b - y\|_2^2$, then $( \hat{b}, \hat{z}, \hat{b}\hat{b}^T)$ (where $\hat{z}_i = 1,\forall i$), is feasible in the semidefinite relaxation above. Note that the objective function above is lower bounded by a strictly convex function of 
$b$ that is independent of $\lambda$, i.e., for any $(b, z, B)$ feasible, 
\begin{align*}
\frac{1}{2} \left\langle X^T X, B \right\rangle - y^T X b + \lambda \sum_{i} z_i + \frac{1}{2}y^T y&\geq \frac{1}{2}\|X b - y\|_2^2 . 
\end{align*}
Since $X^T X \succ 0$, there exists $\tilde{M}$ is sufficiently large such that for all $b$, $\|b\|_\infty \geq \tilde{M}$
\[
\frac{1}{2}\|X b - y\|_2^2  \geq \frac{1}{2} \left\langle X^T X, \hat{b}\hat{b}^T \right\rangle - y^T X \hat{b} + \frac{1}{2}y^T y+ \lambda \sum_{i} \hat{z}_i,
\]
therefore any feasible $(b, z, B)$ with some $b_i \geq \tilde{M}$ cannot be optimal.


In the rest of this paper we consider the following semidefinite relaxation,
\begin{equation}\label{eq:SDP}
\begin{aligned}
\zeta_{SDP} := \min_{b, z, B} \ \ & \frac{1}{2} \left\langle X^T X, B \right\rangle - y^T X b + \frac{1}{2} y^Ty + \lambda \sum_{i} z_i, \\
\ \ \mbox{s.t.}  &   B\succeq bb^T,  \begin{bmatrix} z_i & b_i \\ b_i & B_{ii}\end{bmatrix} \succeq 0,  \forall i. 
\end{aligned}\tag{SDP}
\end{equation}

If the convex constraint $B\succeq bb^T$ were replaced by $B = bb^T$, this is a reformulation of (\ref{eq:l0}). The following proposition can be used to certify when a solution to (\ref{eq:SDP})
also provides a global optimization solution to (\ref{eq:l0}). 
\begin{proposition}\label{prop:rank1}
Assume $\lambda > 0$. Let $(b^*, z^*, B^*)$ be an optimal solution to (\ref{eq:SDP}), then for all $i$, $z_i^* \in [0,1]$. Further, $z_i^*$ takes the value  
$\frac{\left(b_i^*\right)^2}{B^*_{ii}}$ if $B_{ii}^* \neq 0$, and value $0$ otherwise. If $B^*$ is a rank-1 matrix, then $z_i^*$ is binary for all $i$, and $b^*$ is an optimal solution to (\ref{eq:l0}).
\end{proposition}
\begin{proof}
As $z_i$ only appears in the objective and the constraint $\begin{bmatrix} z_i & b_i \\ b_i & B_{ii}\end{bmatrix} \succeq 0$, the smallest value $z_i$ can take 
is $\frac{b_i^2}{B_{ii}}$ if $B_{ii} > 0$, and $0$ otherwise. Note that $B_{ii}\geq b_i^2$ is implied by the constraint $B\succeq bb^T$, if 
$(b^*, z^*, B^*)$ were an optimal solution to (\ref{eq:SDP}), then for all $i$, $z_i^* \in [0,1]$, and $z_i^*$ takes the value  
$\frac{\left(b_i^*\right)^2}{B^*_{ii}}$ if $B_{ii}^* \neq 0$, and value $0$ otherwise.

If $B^*$ is a rank-1 matrix, by $B^* \succeq b^* \left(b^*\right)^T$ we have $B^* = b^* \left(b^*\right)^T$. Therefore $z_i^* = 1$ if $B_{ii}^* \neq 0$, and 
$z_i^*=0$ otherwise. It is then easy to see that
\[
\frac{1}{2} \left\langle X^T X, B^* \right\rangle - y^T X b^* + \lambda \sum_{i} z^*_i + 0.5 y^T y = \frac{1}{2}\left\| X b^* - y \right\|_2^2 + \left\| b^* \right\|_0,
\]
Since (\ref{eq:SDP}) is a relaxation of (\ref{eq:l0}),  
$b^*$ is an optimal solution to (\ref{eq:l0}).
\qed
\end{proof}
Similar to the case of perspective relaxations, (\ref{eq:SDP}) is meaningful only when $X^T X \succ 0$. Otherwise, if $X^T X$ has a nontrivial null space, e.g.,
$\exists y \neq 0, X^T X y = 0$, then by following the recession direction $B \mapsto B+\tau yy^T$, as $\tau \mapsto +\infty$, $B_{ii}$ may become arbitrarily large and $z_i \mapsto 0$ for all 
$i$ such that $y_i \neq 0$. 

The dual problem to (\ref{eq:SDP}) is
\begin{equation}\label{eq:DSDP}
\begin{aligned}
\zeta_{SDP} = \frac{1}{2} y^Ty + \sup_{\epsilon \in \Rbb, \alpha \in \Rbb^p, \delta, t\in \Rbb^{p}} &  \ \ \ 
-\frac{1}{2} \epsilon \\
\mbox{s.t.} & \ \ \  \begin{bmatrix}\epsilon & \alpha^T \\ \alpha & X^T X - \diag(\delta) \end{bmatrix} \succeq 0, 
\begin{bmatrix}\delta_i & t_i \\ t_i & 2\lambda \end{bmatrix} \succeq 0, \forall i, \\
& \ \ \ \alpha_i + \left(X^T y\right)_i + t_i= 0, \ \forall i.
\end{aligned}\tag{DSDP}
\end{equation}
It is easy to see that (\ref{eq:SDP}) is strictly feasible. With the assumption that $X^T X \succ 0$, the dual problem (\ref{eq:DSDP}) is also strictly feasible. 
Therefore strong duality holds and the optimal value is attained at some primal optimal solution $(b^*, z^*, B^*)$ and dual optimal solution 
$(\epsilon^*, \alpha^*, \delta^*,  t^*)$. 

The following theorem shows that we can solve the minimax pair (\ref{eq:minimax})
and (\ref{eq:maxmin}) by solving (\ref{eq:SDP}).
\begin{theorem}\label{thm:minimaxSDP}
Assuming $X^T X \succ 0$, a saddle point for the minimax pair (\ref{eq:minimax}) and (\ref{eq:maxmin}) can be obtained by 
solving the primal-dual pair of semidefinite programs,  (\ref{eq:SDP}) and (\ref{eq:DSDP}). Let  
 $(b^*, z^*, B^*)$  and $(\epsilon^*, \alpha^*, \delta^*, t^*)$ be optimal solutions to (\ref{eq:SDP}) and (\ref{eq:DSDP}), respectively, then
 $(\delta^*, b^*)$ is a saddle point for (\ref{eq:minimax}) and (\ref{eq:maxmin}).
\end{theorem}
\begin{proof}
Let $K(\cdot, \cdot)$ and $C$ be defined as in (\ref{eq:concave_convex}),
 $(b^*, z^*, B^*)$  and $(\epsilon^*, \alpha^*, \delta^*, t^*)$ be optimal solutions to (\ref{eq:SDP}) and (\ref{eq:DSDP}) respectively.
We would like to show that for all $\delta \in C$, $b \in \Rbb^{p}$,
\begin{equation}\label{eq:minmaxGoal}
\max_{\delta \in C} K(\delta, b^*) = \zeta_{SDP} = \min_{b\in \Rbb^p} K(\delta^*, b).
\end{equation}
Provided (\ref{eq:minmaxGoal}),  $(\delta^*, b^*)$ is a saddle point because
\[
K(\delta^*, b^*) \leq \max_{\delta \in C} K(\delta, b^*) = \min_{b\in \Rbb^p} K(\delta^*, b) \leq K(\delta^*, b^*),
\]
which implies
\[
K(\delta, b^*) \leq K(\delta^*, b^*) \leq K(\delta^*, b), \ \ \forall \delta \in C, b\in \Rbb^p.
\]

Note that by our derivation of (\ref{eq:PR}) and (\ref{eq:PRreg}), $\forall \hat{b} \in \Rbb^p, \zeta_{PR(\delta)} = \min_b K(\delta, b) \leq K(\delta, \hat{b})$. So the left and right end of (\ref{eq:minmaxGoal}) satisfy the following conditions,
\[
\max_{\delta \in C} \zeta_{PR(\delta)}\leq \max_{\delta \in C} K(\delta, b^*), \ \ 
\min_{b\in \Rbb^p} K(\delta^*, b) = \zeta_{PR(\delta^*)}.
\]
Therefore to prove (\ref{eq:minmaxGoal}), it suffices to show 
\[
\max_{\delta \in C} \zeta_{PR(\delta)}=\zeta_{SDP}=\zeta_{PR(\delta^*)}.
\]
Firstly, we show that for any admissible $\delta \in C$, $\zeta_{SDP} \geq \zeta_{PR(\delta)}$. In fact, we claim that for any solution $(\bar{b}, \bar{z}, \bar{B})$
feasible in (\ref{eq:SDP}), $(\bar{b}, \bar{z}, \bar{s})$ with $\bar{s} = \diag(\bar{B})$ is a feasible solution to 
(\ref{eq:PR}) with a no-larger objective value. To verify, one has
\[
\begin{aligned}
\frac{1}{2} \langle X^T X, \bar{B}\rangle - y^T X \bar{b} + \lambda \sum_{i} \bar{z}_i
&= \frac{1}{2} \langle X^T X, \bar{b}\bar{b}^T\rangle - y^T X \bar{b} + \lambda \sum_{i} \bar{z}_i + \frac{1}{2} \langle X^T X, \bar{B} - \bar{b}\bar{b}^T\rangle\\
& \geq \frac{1}{2} \langle X^T X, \bar{b}\bar{b}^T\rangle - y^T X \bar{b} + \lambda \sum_{i} \bar{z}_i + \frac{1}{2} \langle \diag(\delta), \bar{B} - \bar{b}\bar{b}^T\rangle\\
&= \frac{1}{2} \bar{b}^T (X^T X - \diag(\delta) ) \bar{b} - (X^T y)^T \bar{b} + \frac{1}{2}\sum_{i }\delta_i \bar{s}_i + 
\lambda \sum_{i} \bar{z}_i.
 \end{aligned}
\]
The inequality is due to the fact that $X^T X - \diag(\delta) \succeq 0$ and $\bar{B} - \bar{b}\bar{b}^T \succeq 0$. Therefore we have 
\[
\max_{\delta \in C} \zeta_{PR(\delta)} \leq \zeta_{SDP}.\]

Now we show that $\zeta_{SDP}\leq \zeta_{PR(\delta^*)}$, which will then complete the proof of (\ref{eq:minmaxGoal}) by
\[
\zeta_{PR(\delta^*)} \leq \max_{\delta \in C} \zeta_{PR(\delta)} \leq \zeta_{SDP} \leq \zeta_{PR(\delta^*)}.
\]
We achieve this by showing
the optimal value of (\ref{eq:DSDP}) is less than or equal to the objective value of any feasible solution
to ($PR(\delta^*)$).
Let $(\bar{b}, \bar{s}, \bar{z})$ denote a feasible solution to ($PR(\delta^*)$),
we have two sets of matrix inequalities
\begin{align}
\label{eq:psd1}\begin{bmatrix}\epsilon^* & \alpha^{*T} \\ \alpha^* & X^T X - \diag(\delta^*) \end{bmatrix} \succeq 0, & \ \ 
\begin{bmatrix}1 & \bar{b}^T \\ \bar{b} & \bar{b}\bar{b}^T\end{bmatrix} \succeq 0, \\
\label{eq:psd2}\begin{bmatrix}\delta_i^* & t_i^* \\ t_i^* & 2\lambda \end{bmatrix} \succeq 0, & \ \ 
\begin{bmatrix}\bar{s}_i & \bar{b}_i \\ \bar{b}_i & \bar{z}_i \end{bmatrix} \succeq 0, \ \ \forall i.
\end{align}
As the inner product between two matrices in (\ref{eq:psd1}) is nonnegative,  we have
\[
\begin{aligned}
\zeta_{SDP} - 0.5y^T y &= -\frac{1}{2} \epsilon^* \leq \alpha^{*T} \bar{b} + \frac{1}{2} \bar{b}\left(X^T X - \diag(\delta^*)\right)\bar{b} \\
& = -y^T X \bar{b} + \sum_{i} \bar{b}_i (-t_i^*) + \frac{1}{2} \bar{b}^T (X^T X - \diag(\delta^*))\bar{b},
\end{aligned}
\]
where the second equality is because of the constraints in (\ref{eq:DSDP}). Next by taking the 
inner product between the matrices in (\ref{eq:psd2}) we obtain
\[
\sum_{i} \bar{b}_i (-t_i^*) \leq \frac{1}{2} \sum_{i}  \delta_i^* \bar{s}_i  + \sum_{i}  \lambda \bar{z}_i.
\]
Therefore, 
\[
\zeta_{SDP}\leq \frac{1}{2} \bar{b}^T (X^T X - \diag(\delta^*))\bar{b}  - y^T X \bar{b} + \frac{1}{2} \sum_{i}  \delta_i^* \bar{s}_i  + \sum_{i}  \lambda \bar{z}_i
+\frac{1}{2}y^T y.
\]
Finally since $(\bar{b}, \bar{s}, \bar{z})$ is an arbitrarily chosen feasible solution, we have $\zeta_{SDP}\leq \zeta_{PR(\delta^*)}$. This completes our proof as previously discussed. \qed
\end{proof}

\begin{remark}\label{rmk:lambdamax}
We provide a remark regarding the computation of $\lambda_{\max}$, i.e., the largest sensible choice of parameter $\lambda$. 
In practice one is often interested in (\ref{eq:l0}) for all $\lambda \geq 0$. If we use (\ref{eq:SDP}) as an approximation to (\ref{eq:l0}), then it is crucial to know the 
smallest penalty parameter $\lambda$ that forces all $z_i$ to be $0$. This number is denoted by $\lambda_{\max}$ and is defined as,  
\[
\lambda_{\max} := \inf\left\{\lambda \ \middle| \ \left(0^{p}, 0^p, 0^{p\times p}\right) \mbox{ is optimal to } (\ref{eq:SDP}) \right\} 
\]
where $0^p$ is the $p\times 1$ zero vector and $0^{p\times p}$ is the $p\times p$ zero matrix. 
We show that $\lambda_{\max}$ can be computed by solving an optimization problem of complexity similar to that of (\ref{eq:SDP}). 
We first prove a checkable condition of when $\left(0^{p}, 0^p, 0^{p\times p}\right)$ is an optimal solution to (\ref{eq:SDP}).
\begin{proposition}
Assuming that $X^T X \succ 0$ and $\lambda > 0$, $\left(0^{p}, 0^p, 0^{p\times p}\right)$ is an optimal solution to (\ref{eq:SDP}) if and only if there exists $\delta \in \Rbb^p$
such that 
\[
X^T X - \diag(\delta) \succeq 0, \ \ \begin{bmatrix}\delta_i & -(X^T y)_i \\ -(X^T y)_i & 2 \lambda \end{bmatrix} \succeq 0, \ \forall i.
\]
\end{proposition}
\begin{proof}
By strong duality of (\ref{eq:SDP}) and (\ref{eq:DSDP}), and complementarity conditions, $\left(0^{p}, 0^p, 0^{p\times p}\right)$ is an optimal solution if and only if there exists $(\epsilon, \alpha, \delta, t)$ feasible in (\ref{eq:DSDP})
such that $\epsilon = 0$. Therefore we must have $\alpha = 0$ and $t = -X^T y$ by the constraints in (\ref{eq:DSDP}), and our conclusion easily follows. \qed
\end{proof}

Then $\lambda_{\max}$ can be computed exactly by solving a semidefinite program:
\begin{equation}\label{eq:lambdaMaxSDP}
\lambda_{\max} = \min \left\{\lambda \middle| \exists \delta \in \Rbb^p,  X^T X - \diag(\delta) \succeq 0, \ \ \begin{bmatrix}\delta_i & -(X^T y)_i \\ -(X^T y)_i & 2 \lambda \end{bmatrix} \succeq 0, \ \forall i 
\right\}.
\end{equation}
\end{remark}

\section{A Probabilistic Interpretation of the Semidefinite Relaxation}\label{sec:prob}
All of our previous derivation of convex relaxations are in a deterministic manner. In this section we provide a probabilistic interpretation of the semidefinite relaxation (\ref{eq:SDP}). 
Especially, our analysis in this section gives insights in interpreting the matrix variable $B$, in addition to the deterministic understanding 
that it is an approximation of outer-product $b b^T$. This is especially useful when (\ref{eq:SDP}) is not an exact relaxation of (\ref{eq:l0}), and an optimal $B$ has high rank. 
Finally, a by-product result in this section shows that (\ref{eq:l0}) can be formulated as a linear program over a convex set related
 to the \textit{Boolean Quadric Polytope}, an important object in polyhedral combinatorics whose facial structure were heavily studied \cite{Padberg89,DezaLaurent1997}.

By considering $\beta$ as the entrywise product of a deterministic vector  and a multivariate Bernoulli random variable, we can
reformulate the deterministic optimization problem (\ref{eq:l0}) into a ``stochastic" form.
Specifically, we denote
\begin{equation}
\beta \in \Bcal_u :=\left\{ u \circ z \ \middle| \ u \in \Rbb^p, z \mbox{ is a discrete random variable defined on } \{0,1\}^p \right\}.
\end{equation}
Then (\ref{eq:l0}) is equivalent to the following stochastic form where one optimizes over $u$ and a class of probabilistic distributions specified by $u$.
\begin{proposition}(\ref{eq:l0}) is equivalent to the following problem,
\begin{align}\label{eq:distri_l0}
\min_{u \in \Rbb^p_{\neq 0}, \beta \in \Bcal_u} \ \ \E_{\beta} \left\{\frac{1}{2} \|X \beta - y \|_2^2 + \lambda \|\beta\|_0\right\},
\end{align}
where $\Rbb_{\neq 0}^p := \left\{u | u_i \neq 0, \forall i \right\}$ is the set of vectors in $\Rbb^p$ that have no component equal to zero, and $\Bcal_u$ denotes all rescaled Bernoulli random vectors where each $\beta_i$ takes value $0$ or $u_i$.  The equivalence is in the following sense: (1) every optimal solution to (\ref{eq:l0}) defines a singleton distribution of 
$\beta$, which is optimal to (\ref{eq:distri_l0}); and (2) every state with positive probability in an optimal solution to (\ref{eq:distri_l0}) is an optimal solution to (\ref{eq:l0}).
\end{proposition}
\begin{proof}
Let $b^* \in \Rbb^p$ be an optimal solution to (\ref{eq:l0}). 
For all $u\in \Rbb_{\neq 0}^p$ and $\beta \in \Bcal_u$,
\[
\E_{\beta} \left\{\frac{1}{2} \|X \beta - y \|_2^2 + \lambda \|\beta\|_0\right\} \geq 
\min_{b \in \Rbb^p}\ \  \frac{1}{2} \|X b - y \|_2^2 + \lambda \|b\|_0 = \frac{1}{2} \|X b^* - y \|_2^2 + \lambda \|b^*\|_0.
\]
Therefore the random variable that takes value $b^*$ with probability 1 is an optimal solution to (\ref{eq:distri_l0}), and the optimal value in (\ref{eq:distri_l0})
equals the optimal value of (\ref{eq:l0}).

On the other hand, let $\beta^* \in \Bcal_{u^*}$ be an optimal solution  to (\ref{eq:distri_l0}), then
\begin{align*}
{\E}_{\beta^*} \left\{\frac{1}{2} \|X \beta^* - y \|_2^2 + \lambda \|\beta^*\|_0\right\} 
&= \sum_{\nu_i \in \{0,u_i^*\}, \nu \in \Rbb^p} P(\beta^* = \nu)
\cdot \left\{\frac{1}{2} \|X \nu - y \|_2^2 + \lambda \|\nu\|_0\right\} \\
& \geq \min_{b \in \Rbb^p}\ \  \frac{1}{2} \|X b - y \|_2^2 + \lambda \|b\|_0.
\end{align*}
As the inequality is actually equality by previous argument, each $\nu$ is an optimal solution to (\ref{eq:l0}) whenever $P(\beta^* = \nu ) > 0$.
%
%
\qed
\end{proof}

Next we show that the correspondence between the objective function of (\ref{eq:distri_l0}) and that of (\ref{eq:SDP}). If we interpret the optimization variables 
$b = \E\beta$ and $B = \E\beta\beta^T$, the linear terms involving $b$ and $B$ in the objective function of (\ref{eq:SDP}) 
is the expected $\ell_2$ loss,
\begin{align*}
\E_\beta \|X\beta -y \|^2_2 = y^T y - 2y^T X\E\beta + \E \left(\beta^T X^T X \beta\right)  &=  \left\langle 
\begin{bmatrix}y^T y & -y^T X \\ -X^T y & X^T X \end{bmatrix}, 
\begin{bmatrix}1 & \E\beta^T \\ \E\beta & \E \beta \beta^T \end{bmatrix} \right\rangle
\\
& =  \left\langle 
\begin{bmatrix}y^T y & -y^T X \\ -X^T y & X^T X \end{bmatrix}, 
\begin{bmatrix}1 & b \\ b & B\end{bmatrix} \right\rangle .
\end{align*}
Therefore by change of variables that $b = \E\beta, B = \E \beta\beta^T$ and $z_i = P(\beta_i \neq 0)$, 
(\ref{eq:distri_l0}) is equivalent to 
\begin{equation}
\label{eq:prob-sdp}
\min_{b, z, B} \ \ \frac{1}{2}\left\langle X^T X, B \right\rangle -  y^T X b + \frac{1}{2}y^T y + \lambda \sum_{i} z_i, \ \ \mbox{s.t.} \ \ (b, z,B) \in \Mcal,
\end{equation}
where set $\Mcal$ is defined as
\begin{equation}\label{eq:setM}
\Mcal := \left\{(b, B, z) \in \Rbb^{p + \frac{p(p+1)}{2} + p} \middle| \exists u \in \Rbb^p_{\neq 0} , \beta \in \Bcal_{u},\  \mbox{s.t.} \ 
b = \E\beta, B = \E\beta\beta^T, z_i = \E \mathbb{1}_{\beta_i \neq 0}, \forall i \right\},
\end{equation}
and $\mathbb{1}_{\beta_i \neq 0}$ is the indicator random variable that takes the value 1 if $\beta_i \neq 0$, and 0 otherwise.

We show $\Mcal$ equals the union of infinitely many rescaled \textit{boolean quadric polytopes} \cite{Padberg89,DezaLaurent1997}.
The boolean quadric polytope ($\mathbf{BQP}$) is one of the most important polytopes studied in combinatorial optimization:
\[
\mathbf{BQP} := \conv\left\{
\left(z, z z^T\right) \in \Rbb^{p+\frac{p(p+1)}{2}}
\middle| z \in \{0,1\}^p\right\}.
\]
 Note that the diagonal of $z z^T$ equals $z$ as $z \in \{0,1\}^p$, and $\mathbf{BQP}$ is  usually
defined in the lower-dimensional space $\Rbb^{\frac{p(p+1)}{2}}$.  We keep the redundancy here for notational convenience.

The following result demonstrates an equivalence between elements of $\mathbf{BQP}$ and all pairs of first and second  moments of multivariate Bernoulli distributions.
\begin{theorem}[Section 5.3 in \cite{DezaLaurent1997}]\label{thm:BernoulliBQP}
Let $(z,Z)$ be a vector in $\Rbb^{p+\frac{p(p+1)}{2}}$, then $(z,Z) \in \mathbf{BQP}$ if and only if there exists a probability space $(\Omega, \Fcal, \mu)$ and events $A_1, ..., A_p \in \Fcal$ such that 
$z_i = \mu(A_i)$ and $Z_{ij}=\mu(A_i \cap A_j)$.
\end{theorem}
The following characterization of $\Mcal$ is then a direct application of Theorem \ref{thm:BernoulliBQP}.
\begin{theorem}\label{thm:setMcal_characterization}
A triplet $(b,B,z) \in \Mcal$ if and only if there is a matrix $Z$ such that $(z,Z) \in \mathbf{BQP}$ and 
$u \in \Rbb^p_{\neq 0}$ such that 
\[
(b,B,z) = \left(z\circ u, Z \circ uu^T, z\right),
\]
where $\circ$ is the Hadamard product of matrices.
Alternatively, $(b,B,z) \in \Mcal$ if and only if $(z,Z) \in \BQP$, where entries of $Z \in \Scal^p$ are defined as 
\[
Z_{ij} = \frac{B_{ij}b_i b_j}{B_{ii} B_{jj}}, \ \mbox{ and } \ \ Z_{ij} = 0 \ \mbox{if} \ \ B_{ii} B_{jj} = 0, \ \forall i,j.
\]
\end{theorem}
\begin{proof}
Suppose that $(b,B,z) = \left(z\circ u, Z\circ uu^T, z\right)$, and $(z,Z) \in \BQP$, we show $(b,B,z) \in \Mcal$. By Theorem \ref{thm:BernoulliBQP} there is a multivariate Bernoulli random vector
$\zeta$ over $\{0,1\}^p$, such that $\E\zeta = z$ and $\E\zeta\zeta^T = Z$. Then $\beta := \zeta \circ u$ is the random vector that proves $(b, B, z) \in \Mcal$.

If $(b,B,z) \in \Mcal$, and let $\beta \in \Bcal_u$ and $u\in \Rbb^p_{\neq 0}$ be the vectors as in (\ref{eq:setM}). Let $\zeta := u^{-1} \circ \beta$, where $u^{-1}$ is a vector 
with entries $u_i^{-1}$, $\forall i$.
Then it is easy to verify that $(b,B,z) = \left(\E\zeta \circ u, \E\zeta\zeta^T \circ uu^T, \E\zeta\right)$, and $(\E\zeta, \E\zeta\zeta^T) \in \BQP$ by Theorem 
\ref{thm:BernoulliBQP}.

Since $t_i = T_{ii}$ for all $(t, T) \in \BQP$, let $(b,B,z) = (z \circ u, Z \circ uu^T, z) \in \Mcal$ and $(z,Z) \in \BQP$, the vector $u_i$ can be determined as,
\begin{equation}\label{eq:ui}
u_i = \begin{cases}\frac{B_{ii}}{b_i}, & \ if \ b_i \neq 0, \\
1, \ & \ if \ b_i = 0,\end{cases}
\end{equation}
and the representation of $Z_{ij}$ easily follows the relation $B_{ij} = Z_{ij}u_i u_j$. \qed
\end{proof}
As the objective function in (\ref{eq:prob-sdp}) is linear, it suffices to optimize the objective function over $\overline{\conv\Mcal}$. We show that, the feasible region of (\ref{eq:SDP}) can be seen 
as a reasonable convex relaxation of $\overline{\conv\Mcal}$, in the sense that they coincide under some projections.
\begin{theorem}
Let $\Mcal$ be a set as defined in (\ref{eq:setM}), we have
\begin{align}
\label{eq:prop1}\overline{\conv\Mcal} &\subseteq \left\{(b,B,z)\in \Rbb^{p+\frac{p(p+1)}{2}+p} \middle| B\succeq bb^T,
\begin{bmatrix}z_i & b_i \\ b_i & B_{ii} \end{bmatrix} \succeq 0, \forall i \right\},\\
\label{eq:prop2}\Proj_{(b,B)} \overline{\conv\Mcal} &= \left\{(b,B)\in \Rbb^{p+\frac{p(p+1)}{2}} \middle| B\succeq bb^T\right\}, \\
\label{eq:prop3}\Proj_{(b,\diag(B),z)} \overline{\conv\Mcal} &= \left\{(b,\diag(B), z)\in \Rbb^{3p} \middle| 
B_{ii}\geq b_i^2, \begin{bmatrix}z_i & b_i \\ b_i & B_{ii} \end{bmatrix} \succeq 0, z_i \leq 1, \forall i\right\}.
\end{align}
\end{theorem}
\begin{proof}
Firstly, the inclusion relation $``\subseteq"$ in (\ref{eq:prop1},\ref{eq:prop2},\ref{eq:prop3}) are all straightforward by the characterization of points in $\Mcal$ as in Theorem \ref{thm:setMcal_characterization}. We only show the other directions for (\ref{eq:prop2}) and (\ref{eq:prop3}).

Since $\left\{(b,B)\in \Rbb^{p+\frac{p(p+1)}{2}} \middle| B\succeq bb^T\right\} = \conv\{(b,bb^T)| b\in\Rbb^p\}$,  
it suffices to show in (\ref{eq:prop2}) that for each $b\in \Rbb^p$, $(b,bb^T) \in \Proj_{(b,B)} \overline{\conv\Mcal}$.
This is true because $(b,bb^T)$ can be written as $\left(t\circ u, (tt^T)\circ(uu^T) \right)$ by taking $(t_i, u_i)=(1, b_i)$ if $b_i\neq 0$ and $(t_i, u_i)=(0, 1)$ if $b_i=0$.
So $(b,bb^T, t) \in \Mcal$ and $(b,bb^T) \in \Proj_{(b,B)} \overline{\conv\Mcal}$ for all $b \in \Rbb^p$.

For (\ref{eq:prop3}), it suffices to show that all extreme points of the right hand set are in $\Proj_{(b,\diag(B),z)} \overline{\conv\Mcal}$. 
Note that all such extreme points are in the form of $\left(b, b\circ b, z\right) \in \Rbb^{3p}$, where for each $i$, either $z_i=1$ or $z_i=b_i=0$. 
Points in this form are projected from $(b, bb^T, z)$, where for each $i$, either $z_i=1$ or $z_i=b_i=0$. It is easy to see that 
all such points $(b, bb^T, z)$ are in $\Mcal$ by Theorem \ref{thm:setMcal_characterization}. \qed
\end{proof}
We remark that that the inequality $z_i \leq 1$ in (\ref{eq:prop3}) is redundant in (\ref{eq:SDP}) by Proposition \ref{prop:rank1}.

\section{Randomized Rounding by the Goemans-Williamson Procedure}\label{sec:GWrounding}
In this section we show the analogy that (\ref{eq:SDP}) is to (\ref{eq:l0}) as a semidefinite relaxation is to the Max-Cut problem. The semidefinite
relaxation for Max-Cut under consideration was proposed and analyzed by Goemans and Williamson \cite{GoWi94}, and Nesterov \cite{Nesterov97}. We show 
(\ref{eq:l0}) can be reformulated as a two-level problem, whose inner problem is a Max-Cut problem. Then (\ref{eq:SDP}) can be realized by
replacing the inner problem with its semidefinite relaxation. This observation suggests to apply Goemans-Williamson rounding to (\ref{eq:SDP}), in 
order to generate approximate solutions to (\ref{eq:l0}).

We use $\zeta_{L0}$ to denote the optimal value of (\ref{eq:l0}), and $\zeta_{SDP}$ is the optimal value of (\ref{eq:SDP}). Using similar technique 
as in Section \ref{sec:prob}, we redefine $\beta = u\circ z$, where $u\in \Rbb^p$ and $z \in \{0,1\}^p$. For a fixed vector $u \in \Rbb^p$, we define the following binary 
quadratic program,
\begin{equation}\label{eq:BQP}
\zeta_{BQP(u)} := \min_{z \in \{0,1\}^p} \ \ \frac{1}{2} \left\|X \diag(u) z - y\right\|_2^2 + \lambda \sum_{i=1}^p z_i. \tag{BQP(u)}
\end{equation}
Consider $\zeta_{BQP(u)}$ as a function of $u$, then we have 
\begin{equation}\label{eq:L0MaxCut}
\zeta_{L0} = \min_{u \in \Rbb^p} \ \ \zeta_{BQP(u)}. 
\end{equation}
 
 
It is well-known that binary quadratic programs can be reformulated as Max-Cut problems. We explicitly state the reformulation here.  We define matrix 
 \begin{align*}
 Q(u) & := \begin{bmatrix}1 & 0^T \\ 0 & \diag(u) \end{bmatrix}
  \begin{bmatrix}y^T y \ \ & -y^T X  \\ - X^T y \ \ \ \  &  X^T X \end{bmatrix}
  \begin{bmatrix}1 & 0^T \\ 0 & \diag(u) \end{bmatrix} +  \begin{bmatrix}0 & 0^T \\ 0 & 2 \lambda I \end{bmatrix}
  \\
 & = \begin{bmatrix}y^T y \ \ & -y^T X \diag(u) \\ -\diag(u) X^T y \ \ \ \  & \diag(u) X^T X \diag(u) + 2\lambda I\end{bmatrix}.
 \end{align*}
Then (\ref{eq:BQP}) is equivalent to 
\[
\zeta_{BQP(u)} = \min_{z \in \{0,1\}^p} \ \ \frac{1}{2} \left\langle Q(u), \begin{bmatrix}1 & z^T \\ z & zz^T \end{bmatrix}\right\rangle.
\]
 By change of variables
 \[
 t \longleftarrow \begin{bmatrix}1 & 0^T \\ -e & 2I\end{bmatrix}
\begin{bmatrix}1 \\ z\end{bmatrix}, ~~~~
t t^T \longleftarrow \begin{bmatrix}1 & 0^T \\ -e & 2I\end{bmatrix}
\begin{bmatrix}1 & z^T \\ z & z z^T \end{bmatrix} \begin{bmatrix}1 & -e^T \\ 0 & 2I\end{bmatrix},
 \]
(\ref{eq:BQP}) can be reformulated as a Max-Cut problem:
 \[
 \zeta_{BQP(u)} = \min_{t \in \{-1,1\}^p} \ \ \frac{1}{2} \left\langle \begin{bmatrix}1 & -e^T \\ 0 & 2I\end{bmatrix} Q(u) 
  \begin{bmatrix}1 & 0^T \\ -e & 2I\end{bmatrix}, t t^T\right\rangle.
 \]
Therefore a semidefinite relaxation for (\ref{eq:BQP}) is 
  \begin{equation}\label{eq:MCSDP}
 \zeta_{MCSDP(u)} := \min_{T \in \Scal^{n+1}_+, T_{ii}=1, \forall i} \ \ \frac{1}{2} \left\langle \begin{bmatrix}1 & -e^T \\ 0 & 2I\end{bmatrix} Q(u) 
  \begin{bmatrix}1 & 0^T \\ -e & 2I\end{bmatrix}, T\right\rangle. \tag{MCSDP(u)}
 \end{equation}
 
In order to show the connection between (\ref{eq:SDP}) and (\ref{eq:MCSDP}), we need the following lemma:
 \begin{lemma}
 Let $T \in \Scal^{p+1}$, define 
 \[
\begin{bmatrix}1 & z^T \\ z & Z\end{bmatrix}  := 
  \begin{bmatrix}1 & 0^T \\ -e & 2I\end{bmatrix}^{-1} T
  \begin{bmatrix}1 & -e^T \\ 0 & 2I\end{bmatrix}^{-1}
  \]
  Then $T\succeq 0$, $T_{ii}=1, \forall i=1,...,p+1$ if and only if  $ Z \succeq zz^T$, and $Z_{ii}=z_i, \forall i=1,...,p$.
 \end{lemma}
 \begin{proof}
Since the matrix $\begin{bmatrix}1 & 0^T \\ -e & 2I\end{bmatrix}$ is invertible, $T\succeq 0$ if and only if  $Z \succeq zz^T$.
 If we denote $ T := \begin{bmatrix}T_{11} & \tilde{t}^T \\ \tilde{t} & \tilde{T}\end{bmatrix} $,
 then 
 \[
 \begin{bmatrix}1 & z^T \\ z & Z\end{bmatrix}  = \begin{bmatrix}T_{11} & \frac{\tilde{t}^T + T_{11} e^T}{2} \\
  \frac{\tilde{t}^T + T_{11} e^T}{2} & \frac{\tilde{T} + \tilde{t} e^T + e \tilde{t}^T + T_{11} ee^T}{4}\end{bmatrix}.
 \]
 It is then straightforward to check that $T_{ii} = 1, 1\leq i\leq p+1$ if and only if $Z_{ii} = z_i, 1\leq i\leq p$.
 \end{proof}
 Therefore \ref{eq:MCSDP} can be reformulated as
  \begin{align}
 \zeta_{MCSDP(u)} &= \min_{Z\succeq zz^T, Z_{ii} = z_i,\forall i} \ \ \frac{1}{2} \left\langle Q(u), \begin{bmatrix}1 & z^T \\ z & Z\end{bmatrix} \right\rangle
 \notag \\
 &= \min_{Z\succeq zz^T, Z_{ii} = z_i,\forall i} \ \ \frac{1}{2} \left\langle  \begin{bmatrix}y^T y \ \ & -y^T X  \\ - X^T y \ \ \ \  &  X^T X \end{bmatrix},
\begin{bmatrix}1 & 0^T \\ 0 & \diag(u) \end{bmatrix}  \begin{bmatrix}1 & z^T \\ z & Z\end{bmatrix} \begin{bmatrix}1 & 0^T \\ 0 & \diag(u) \end{bmatrix}
\right\rangle + \lambda e^T z. \label{MCSDP-uz}
 \end{align}
 The following theorem proves a relation between (\ref{eq:SDP}) and (\ref{eq:MCSDP}), in parallel to equation (\ref{eq:L0MaxCut}).
 \begin{theorem}\label{thm:SDP-MC}
 We have
 \[
 \zeta_{SDP} = \min_{u \in \Rbb^p} \ \ \zeta_{MCSDP(u)}
 \]
 Let $(b^*, B^*, z^*)$ be an optimal solution to (\ref{eq:SDP}) with $\lambda > 0$, then an optimal $u$ is attained at $u^*$ where
 \[
u^*_i := \begin{cases}\frac{B^*_{ii}}{b^*_i}, & \ if \ b^*_i \neq 0 \\
1, \ & \ if \ b_i = 0\end{cases} \ \ \forall 1\leq i \leq p.
 \]
 \end{theorem}
 \begin{proof}
 First we show that $\zeta_{SDP} \leq \zeta_{MCSDP(u)}$ for all $u\in \Rbb^p$. This is because if  $(z,Z)$ is feasible in (\ref{MCSDP-uz}), and we define
 \[
\begin{bmatrix}1 & b^T \\ b & B \end{bmatrix} :=
\begin{bmatrix}1 & 0^T \\ 0 & \diag(u) \end{bmatrix}  
\begin{bmatrix}1 & z^T \\ z & Z\end{bmatrix} 
\begin{bmatrix}1 & 0^T \\ 0 & \diag(u) \end{bmatrix}
 \]
Then $(b,B,z)$ is feasible in (\ref{eq:SDP}) with the same objective value.
 
On the other hand, suppose that $(b^*, B^*, z^*)$ is an optimal solution to (\ref{eq:SDP}), and define 
\begin{equation}\label{eq:genZ}
u^*_i = \begin{cases}\frac{B^*_{ii}}{b^*_i}, & \ if \ b^*_i \neq 0 \\
1, \ & \ if \ b_i = 0\end{cases} \ \ \forall 1\leq i \leq p, \ \mbox{ and } \ 
Z^*_{ij} = \begin{cases}
\frac{B^*_{ij} b_i^* b_j^*}{B^*_{ii} B^*_{jj}},  & \mbox{ if } B_{ii} B_{jj} \neq 0, \\
0, & \mbox{ if } B_{ii} B_{jj} = 0,
\end{cases}
\ \ \  \forall 1\leq i,j\leq p. 
\end{equation}
Then we claim that $(z^*, Z^*)$ is feasible in (\ref{MCSDP-uz}) with $u=u^*$, which proves $\zeta_{SDP} \geq \zeta_{MCSDP(u^*)}$. Indeed, by
Proposition \ref{prop:rank1}, $z^*_i = 0$ whenever $B_{ii}^* = 0$. For all $1\leq i, j \leq p$ such that $B_{ii} B_{jj} \neq 0$,
\[
\left(Z^* - z^* \left( z^* \right)^T \right)_{ij} = \frac{B^*_{ij} b_i^* b_j^*}{B^*_{ii} B^*_{jj}} - \frac{\left(b_i^* b_j^*\right)^2}{B^*_{ii} B^*_{jj}}
=\left(B_{ij}^* - b_i^* b_j^*\right) \frac{B^*_{ij} b_i^* b_j^*}{B^*_{ii} B^*_{jj}}.
\]
Therefore $Z^* - z^* \left( z^*\right)^T$ is the Hadamard product of two positive semidefinite matrices restricted to the rows/columns in set 
$\left\{i | B_{ii}^* > 0\right\}$, and is positive semidefinite. Further again by Proposition \ref{prop:rank1}, $Z_{ii}^* = \frac{b_i^2}{B_{ii}} = z^*_i$ 
if $B_{ii}^* \neq 0$ and $Z_{ii}^* = 0 = z^*_i$ if $B_{ii}^* = 0$. \qed
 \end{proof}
 
Motivated by the relations (\ref{eq:L0MaxCut}) and Theorem \ref{thm:SDP-MC}, given an optimal solution $(b^*, B^*, z^*)$ to (\ref{eq:SDP}),
we may interpret it as an optimal solution to (\ref{eq:MCSDP}) with $u=u^*$  (where $u^*$ is defined as in Theorem \ref{thm:SDP-MC}). 
Then we may construct an approximate solution to (\ref{eq:BQP}) with $u=u^*$ 
using Goemans-Williamson rounding, and reconstruct an approximate solution to (\ref{eq:l0}). 
This rounding procedure is described in Algorithm \ref{alg:round}.
\RestyleAlgo{boxruled}
\LinesNumbered
\begin{algorithm}[t]
\caption{A randomized rounding algorithm for (\ref{eq:SDP}) \label{alg:round}}
    \KwData{Data matrices $(X,y)$, an optimal solution to (\ref{eq:SDP}), denoted by $(b^*, B^*, z^*)$, and a sample size $N$;}
 \KwResult{An approximate solution $\hat{b}$ to (\ref{eq:l0});}
  Generate matrix $Z^* \in \Scal^p $ by (\ref{eq:genZ}); \\
  Generate matrix $T^* \in \Scal^{p+1}$ by 
 \[T^* :=  \begin{bmatrix}1 & 0^T \\ -e & 2I\end{bmatrix}
\begin{bmatrix}1 & \left(z^*\right)^T \\ z^* & Z^* \end{bmatrix} \begin{bmatrix}1 & -e^T \\ 0 & 2I\end{bmatrix};
\] \\
Compute a factorization $T^* = U^* \left( U^* \right)^T$, where $U^* \in \Rbb^{p\times r}$; \\
Randomly generate vectors $\left\{v^{(1)}, ..., v^{(N)} \right\} \subseteq \Rbb^r$ from the normal distribution with mean $0$ and covariance matrix $I$; \\
For each $k=1,...,N$, compute $t^{(k)} \leftarrow \sgn(U^* r)$; If $t^{(k)}_1 = -1$, $t^{(k)} \leftarrow -t^{(k)}$; \\
For each $k=1,...,N$, compute vector $z^{(k)} \in \{0,1\}^p$ by
\[
z^{(k)}_j \leftarrow  0.5 \left( t^{(k)}_{j+1} + 1\right), j=1,...,p; 
\] \\
For each $k=1,...,N$, 
compute $\nu^{(k)}$ by solving a linear regression problem in a restricted subspace
\[
\nu^{(k)} = \lambda \sum_j z^{(k)} + 0.5* \min_{b \in \Rbb^{p}} \ \ \left\{ \left\|X b - y \right\|_2^2 \ \middle| \ b_j = 0\  \forall j \mbox{ s.t., } z_j^{(k)} = 0\right\} ; 
\]
let $b^{(k)}$ denote an optimal solution; \\
Let $K$ be the index such that $\nu^{(K)}$ is the smallest in $\left\{\nu^{k}, k=1,...,N\right\}$.
Then the output vector is set as $\hat{b} := b^{(K)}$.
\end{algorithm}

\section{Numerical Results}
We perform preliminary numerical experiments on simulated data sets. Our results show that (\ref{eq:SDP}) is a much tighter relaxation than a convex relaxation
proposed in \cite{PilanciWainwrightGhaoui2015}. We also conduct experiments to show the effectiveness of our rounding algorithm proposed in section \ref{sec:GWrounding}.

We consider the formulation (\ref{eq:l2l0}),
\[
\min_b \ \ \frac{1}{2} \|X b - y\|_2^2 + \frac{1}{2} \mu \|b\|_2^2 + \lambda \|b\|_0.
\]
We have shown in section \ref{sec:pp} (Remark \ref{rmk:PWG}) that the convex relaxation (\ref{eq:PWGpenalty}) (proposed by Pilanci, Wainwright and Ghaoui \cite{PilanciWainwrightGhaoui2015}) 
can be derived as a special case of perspective relaxation. So the semidefinite relaxation proposed in section \ref{sec:SDP}, when applied to the equivalent form (\ref{eq:reformL2L0}), is theoretically no weaker 
than (\ref{eq:PWGpenalty}). In the following example we show that (\ref{eq:SDP}) is indeed much tighter on our simulated problem sets. For comparison, we also solve the MIQP formulation of (\ref{eq:l2l0}) with Gurobi,
\begin{equation}\label{eq:MIQPtest}
\min_b \ \ \frac{1}{2} \|X b - y\|_2^2 + \frac{1}{2} \mu \|b\|_2^2 + \lambda \sum_{i=1}^p z_i, \ \ s.t. \ \ \ -M z_i \leq b_i \leq M z_i, \ \ \ z \in \{0,1\}^p.
\end{equation}

In the following example, we set $n=100$, $p=60$, and the ``true" sparsity level $k=10$. Each row of $X$ is randomly generated with the normal distribution $\Ncal(0, I)$, where $I$ is the $p\times p$ identity matrix, and then
 divided by $\sqrt{n}$ for normalization. An underling true sparse vector $b^{true} \in \Rbb^p$ is generated by 
\[
b^{true}_i  = \begin{cases} U_{[-1, -0.5]\cup [0.5, 1]}, & i=1,...,k, \\ 0, & i=k+1,...,p, \end{cases}
\]
where $U_{S}$ is the uniform distribution on set $S$.  Then the response vector $y$ is generated by 
\[
y = X b^{true} + \epsilon, \ \ \mbox{ where } \ \epsilon_i \sim \Ncal(0,5), \ \forall i.
\]
When solving MIQP formulation (\ref{eq:MIQPtest}), $M$ is set as $5 \|b^{true}\|_\infty$.

For each pair of parameters $(\lambda, \mu)$, we randomly generate 30 instances. For each instance, we run Gurobi for 60 seconds, and denote the best upper bound as ${\tau}_{UB}$, and the best lower bound as $\tau_{Grb}$.
The optimal values of convex relaxation (\ref{eq:PWGpenalty}), as well as (\ref{eq:SDP}) applied to (\ref{eq:reformL2L0}), are computed and denoted by $\tau_{PWG}$ and $\tau_{SDP}$, respectively. Then three kinds of relative gap 
are computed by 
\[
\mbox{GrbGap} = \frac{\tau_{UB} - \tau_{Grb}}{\tau_{UB}} \times 100\%, \ \ \mbox{SDPGap} = \frac{\tau_{UB} - \tau_{SDP}}{\tau_{UB}} \times 100\%, \ \ \mbox{PWGGap} = \frac{\tau_{UB} - \tau_{PWG}}{\tau_{UB}} \times 100\%.
\]
Table \ref{tab:gaps} summarizes the average relative gap of the 30 instances for each pair of $\lambda$ and $\mu$. All experiments are run on a workstation with AMD Opteron(tm) Processor 6344,
which has a max clock speed 2.6GHz and 24 cores. 

\begin{table}[htdp]
\caption{Average relative gap of Gurobi, (\ref{eq:SDP}) and convex relaxation proposed in \cite{PilanciWainwrightGhaoui2015}} \label{tab:gaps}
\begin{center}
\begin{tabular}{c|r|rrrrr|}
\hline
\multicolumn{2}{c}{} & $\mu=0.1$& $\mu=0.2$& $\mu=0.3$& $\mu=0.4$& $\mu=0.5$ \\
\hline
 \multirow{4}{*}{$\lambda=0.1$} 
&SDPGap&   2.29\%&    1.28\%&    0.72\%&    0.56\%&    0.36\%    \\
&PWGGap &   7.97\%&  4.20\%&    2.79\%&    2.06\%&    1.55\%    \\
\cline{2-7}
&GrbGap&   4.88\%&    4.09\%&    4.00\%&    4.09\%&    3.91\%     \\
& (\#nodes) & (7.8E5) & (7.9E05) & (7.5E05) & (7.2E5) & (7.2E5)  \\
 \hline 
 \multirow{4}{*}{$\lambda=0.2$} 
 &SDPGap &    3.77\%&    2.15\%&    1.43\%&    0.88\%&    0.65\%  \\
&PWGGap &     12.25\%&    7.12\%&    4.81\%&    3.29\%&    2.76\%  \\
\cline{2-7}
&GrbGap &   4.17\%&   4.24\%&    3.33\%&    3.19\%&    3.03\%   \\
&(\#nodes) & (7.8E5) & (7.3E5) & (7.1E5) & (6.4E5) & (6.2E5)  \\
 \hline
  \multirow{4}{*}{$\lambda=0.3$} 
&SDPGap &     4.55\%&    2.79\%&    1.49\%&    0.98\%&    0.82\%   \\
&PWGGap &      14.19\%&    8.90\%&    5.42\%&    3.93\%&    3.30\%    \\
\cline{2-7}
&GrbGap &     1.62\%&    2.07\%&    1.09\%&    1.09\%&    1.60\%    \\
&(\#nodes) &(5.3E5) & (5.3E5) & (4.3E5) & (5.4E5) & (5.6E5)  \\
\hline
 \multirow{4}{*}{$\lambda=0.4$} 
 &SDPGap &      5.13\%&    2.76\%&    1.50\%&    0.91\%&    0.68\%    \\
&PWGGap &      15.98\%&    9.26\%&    6.01\% &   4.12\%&    3.24\%    \\
\cline{2-7}
& GrbGap &      0.74\%&    0.65\%&    0.11\%&    0.00\%&    0.04\%    \\
&(\#nodes) &(4.2E5) & (2.6E5) & (2.5E5) & (1.6E5) & (1.2E5)  \\
 \hline
  \multirow{4}{*}{$\lambda=0.5$} 
  &SDPGap &      4.60\%&    2.53\%&    1.59\%&    0.89\% &   0.67\% \\
&PWGGap &      15.49\%&   9.02\%&    6.11\%&    4.27\% &   3.28\% \\
\cline{2-7}
& GrbGap &      0.01\%&    0.00\%&    0.00\%&    0.00\%&   0.00\% \\
&(\#nodes)& (1.6E5) & (9.8E4) & (9.3E4) & (7.8E4) & (6.7E4)  \\
\hline
\end{tabular}
\end{center}
\label{tab:gap}
\end{table}%

In all cases, SDPGap is much smaller than PWGGap. In general, the two convex relaxations (\ref{eq:SDP}) and PWG relaxation become tighter as $\mu$ gets larger. This is expected as perspective
relaxation performs especially good when corresponding quadratic forms are nearly diagonal.
As $\lambda$ gets larger, the lower bounds obtained by Gurobi in the 60-second time limit improves, and when $\lambda \geq 0.3$, they become better
than the other two convex relaxations. However, this lower bounds are obtained at the expense of examining large numbers ($10^4 \sim 10^5$) of nodes, while
the computational costs for solving (\ref{eq:SDP}) and the PWG relaxation are negligible in our setting of small $p (=60)$. 

We now consider the effectiveness of Goemans-Williamson rounding in our context. We applied Algorithm \ref{alg:round} to the SDP solutions for all 750 generated instances with sample size $N=1000$.
Let $\tau_{GW}$ denotes the best objective value of problem (\ref{eq:l2l0}) found by the rounding procedure. In majority of cases we have $\tau_{GW} \geq \tau_{UB}$, i.e., the upper bounds 
obtained by the rounding procedure are no better than those found by Gurobi in the time limit of 60 seconds. However in 555 out of the 750 instances 
 they are equal, i.e.,  $\tau_{GW} = \tau_{UB}$. There are only 6 instances where the rounding procedure provides strictly better upper bounds.
However $\tau_{GW}$ is always very close to $\tau_{UB}$. In table \ref{tab:UB}, we report the averaged relative differences 
\[
\frac{\tau_{GW} - \tau_{UB}}{\tau_{UB}} \times 100 \%
\]
for each pair of choices of $\lambda$ and $\mu$. We also ran Gurobi for a longer period of time (300 seconds) on a subset of instances. In all instances we tested, Gurobi reports no improvement on the 
upper bounds after the first 60 seconds.

\begin{table}[htdp]
\caption{Relative difference of upper bounds by Goemans-Williamson rounding and  Gurobi}\label{tab:UB}
\begin{center}
\begin{tabular}{c|rrrrr|}
\hline
 &$\mu=0.1$& $\mu=0.2$& $\mu=0.3$& $\mu=0.4$& $\mu=0.5$ \\
\hline
$\lambda = 0.1$ &   0.21\% & 0.07\%&   0.02\%&    0.01\%&    0.00\%    \\
\hline
$\lambda = 0.2$ &    0.28\%  &  0.13\% &   0.02\%  &  0.01\% &   0.00\% \\
$\lambda = 0.3$  &  0.30\%  &  0.14\%  &  0.06\%  &  0.00\%  &  0.00\% \\
$\lambda = 0.4$  &  0.34\%  &  0.03\%  &  0.03\%  &  0.01\%  &  0.00\% \\
$\lambda = 0.5$   & 0.20\% &   0.09\%  &  0.03\% &   0.00\% &   0.00\% \\
\hline
\end{tabular}
\end{center}
\label{tab:gap}
\end{table}%

We finally comment on the computational cost of solving (\ref{eq:SDP}). The size of (\ref{eq:SDP}) is primarily determined by $p$ -- the number of predictor variables in regression, 
while does not depend on $n$. Also note that (\ref{eq:SDP}) has a relatively ``clean" form, i.e., the number of linear constraints is small, and in fact grows linearly with respect to $p$. 
The dual-scaling interior point algorithm for SDP \cite{BensonYeZhangDsdp1} is especially suitable for solving such SDP problems to high accuracy. In table  \ref{tab:SDPtime} we report the 
typical computational time needed to solve one instance of (\ref{eq:SDP}) as $p$ increases in table, using the software DSDP \cite{dsdp5} implemented by Benson, Ye and Zhang, 
with their default parameters.
\begin{table}[htdp]
\caption{Computational time (seconds) to solve (\ref{eq:SDP}) with DSDP \cite{dsdp5}}\label{tab:SDPtime}
\begin{center}
\begin{tabular}{|ccccc|}
\hline
$p=50$ & $p=100$ & $p=200$ & $p=400$ & $p=800$ \\
\hline
0.33 & 1.20 & 4.58 & 37.48 & 278.9 \\
\hline
\end{tabular}
\end{center}
\label{tab:gap}
\end{table}%

In practice (\ref{eq:SDP}) needs to be solved many times for different choices of $\lambda$. Therefore when $p \geq 400$, it may not be a viable solution to solve (\ref{eq:SDP}) using interior point 
methods. In such cases, it makes sense to consider cheaper approximate algorithms, such as the first-order algorithms, that also benefit from warm-starting when $\lambda$ is 
slightly changed. An especially attractive approach is to use low rank factorizations and nonlinear programming \cite{BurMon03-1}. We will leave comprehensive computational studies for future work.

\section{Conclusions}
One of the most popular approaches for sparse regression is to use various convex or nonconvex penalty functions to approximate the $\ell_0$ norm.
In this paper, we propose an alternative perspective by considering convex relaxations for the mixed-integer quadratic 
programming (MIQP) formulations of the sparse regression problem. We show that convex relaxations, especially conic optimization, can be 
a valuable tool. Both of the minimax concave penalty (MCP) function and the reverse huber penalty function considered in the literature are special cases of  
perspective relaxation for the MIQP formulation. The tightest perspective relaxation leads to a minimax problem that can be solved by semidefinite programming. 
This semidefinite relaxation has several elegant interpretations. First,  it achieves the balance of convexity and the approximation quality to the $\ell_0$ norm in a minimax sense. Second, 
it can be interpreted as searching for the first two moments of a \textit{rescaled multivariate Bernoulli random variable} that is used to represent our ``beliefs" of parameters in estimation, 
which then reveals connections with the Boolean Quadric Polytope in combinatorial optimization.
Third, by interpreting the sparse regression problem as a two level optimization with the inner level being the Max-Cut problem, our proposed semidefinite relaxation can be realized by 
replacing the inner level problem with its semidefinite relaxation considered by Goemans and Williamson. 
The last interpretation suggests to adopt Goemans-Williamson rounding procedure to find approximate solutions to the sparse regression problem. Preliminary numerical experiments 
demonstrate our proposed semidefinite relaxation is much tighter than a convex relaxation proposed by Pilanci, Wainwright and El Ghaoui using the reverse Huber penalty 
 \cite{PilanciWainwrightGhaoui2015}. The effectiveness
of Goemans-Williamson rounding is also demonstrated.

Future work should include a more comprehensive simulation study to compare the SDP-based variable selection method with other convex and nonconvex penalization-based methods,
in terms of their support identification and prediction accuracy. Algorithmically, it is of interests to develop more scalable algorithms to approximately solve the semidefinite relaxation (\ref{eq:SDP}) by exploiting 
the (relatively simple) problem structure. 

\bibliographystyle{spmpsci}      
\bibliography{regrelax}   

\end{document}